\documentclass{article} %
\usepackage{iclr2026_conference,times}

\usepackage{amsmath,amsfonts,bm}

\def\eqref#1{equation~\ref{#1}}

\def\1{\bm{1}}

\def\eps{{\epsilon}}

\DeclareMathAlphabet{\mathsfit}{\encodingdefault}{\sfdefault}{m}{sl}
\SetMathAlphabet{\mathsfit}{bold}{\encodingdefault}{\sfdefault}{bx}{n}

\def\gB{{\mathcal{B}}}

\def\gD{{\mathcal{D}}}

\def\gF{{\mathcal{F}}}

\def\gH{{\mathcal{H}}}
\def\gI{{\mathcal{I}}}

\def\gL{{\mathcal{L}}}

\def\gN{{\mathcal{N}}}
\def\gO{{\mathcal{O}}}

\def\gR{{\N^2}}

\def\gT{{\mathcal{T}}}

\def\gV{{\mathcal{V}}}

\def\gX{{\mathcal{X}}}
\def\gY{{\mathcal{Y}}}
\def\gZ{{\mathcal{Z}}}

\newcommand{\R}{\mathbb{R}}

\usepackage{amsthm}

\usepackage{pgfplots}

\newtheorem{thm}{Theorem}
\newtheorem*{thm*}{Theorem}
\newtheorem{lemma}{Lemma}
\newtheorem{cor}{Corollary}
\newtheorem*{cor*}{Corollary}
\newtheorem{defn}{Definition}
\newtheorem{prop}{Proposition}
\newtheorem*{prop*}{Proposition}

\newcommand{\N}{\mathbb{N}}

\newcommand{\leqc}{\mathrel{\mathchoice%
  {\overset{+}{\leq}}%
  {\text{\scriptsize $\overset{+}{\leq}$}}%
  {\text{\scriptsize $\overset{+}{\leq}$}}%
  {\text{\scriptsize $\overset{+}{\leq}$}}%
}}

\newcommand{\geqc}{\stackrel{+}{\geq}}
\newcommand{\eqc}{\stackrel{+}{=}}

\newcommand{\B}{\{0,1\}^*}

\newcommand{\bigast}{\raisebox{-0.3ex}{\scalebox{1.2}{\ensuremath{\ast}}}}

\usepackage{stmaryrd} 

\usepackage{booktabs} 

\usepackage{siunitx}  

\usepackage{titletoc}

\newcommand{\eat}[1]{}

\newcommand{\zmap}{\texttt{zmap}}

\usepackage{wrapfig}    

\usepackage{tikz}
\usetikzlibrary{positioning}

\usepackage{minted}

\usepackage{amsmath}
\usepackage{blkarray} 

\usepackage{enumitem}

\DeclareMathOperator{\im}{im}
\DeclareMathOperator{\dom}{dom}

\definecolor{gdmblue900}{HTML}{174EA6}
\definecolor{gdmblue800}{HTML}{185ABC}
\definecolor{gdmblue700}{HTML}{1967D2}
\definecolor{gdmblue600}{HTML}{1A73E8}
\definecolor{gdmblue500}{HTML}{4285F4}
\definecolor{gdmblue400}{HTML}{669DF6}
\definecolor{gdmblue300}{HTML}{8AB4F8}
\definecolor{gdmblue200}{HTML}{AECBFA}
\definecolor{gdmblue100}{HTML}{D2E3FC}

\usepackage{pifont}

\newcommand{\greencheck}{\textcolor{green!70!black}{\ding{51}}}
\newcommand{\redcross}{\textcolor{red!70!black}{\ding{55}}}

\usepackage{makecell} 

\usepackage[colorlinks,allcolors=gdmblue700]{hyperref}

\usepackage{url}

\title{Bridging Kolmogorov Complexity and Deep Learning: Asymptotically Optimal Description Length Objectives for Transformers}

\author{
Peter Shaw \\
Google DeepMind \\
\And
James Cohan \\
Google Research \\
\And
Jacob Eisenstein \\
Google DeepMind \\
\And 
Kristina Toutanova \\
Google DeepMind \\
}

\iclrfinalcopy %

\begin{document}

\maketitle

\begin{abstract}
The Minimum Description Length (MDL) principle offers a formal framework for applying Occam's razor in machine learning. However, its application to neural networks such as Transformers is challenging due to the lack of a principled, universal measure for model complexity.
This paper introduces the theoretical notion of asymptotically optimal description length objectives, grounded in the theory of Kolmogorov complexity. We establish that a minimizer of such an objective achieves optimal compression, for any dataset, up to an additive constant, in the limit as model resource bounds increase.
We prove that asymptotically optimal objectives exist for Transformers, building on a new demonstration of their computational universality. We further show that such objectives can be tractable and differentiable by constructing and analyzing a variational objective based on an adaptive Gaussian mixture prior.
Our empirical analysis shows that this variational objective selects for a low-complexity solution with strong generalization on an algorithmic task, but standard optimizers fail to find such solutions from a random initialization, highlighting key optimization challenges. More broadly, by providing a theoretical framework for identifying description length objectives with strong asymptotic guarantees, we outline a potential path towards training neural networks that achieve greater compression and generalization.
\end{abstract}

\section{Introduction}

The principle of Occam's razor---that simpler explanations are preferable---is a foundational concept in machine learning. Algorithmic information theory provides an elegant formalization through the \emph{Minimum Description Length}~\citep[MDL;][]{rissanen1978modeling} principle. The MDL principle frames learning as data compression: the best model for a dataset is the one that minimizes the combined length of the model's description plus the description of the data encoded with that model. Intuitively, any regularity in the data that is useful for prediction is also useful for compression, and vice versa.

The empirical success of large neural networks might seem to contradict the MDL principle if we consider a naive, uniform encoding of their weights. Such networks often generalize surprisingly well, even when highly over-parameterized and trained with simple maximum likelihood objectives~\citep{zhang2017understanding}. However, various methods have been proposed for training compressible neural networks based on, e.g., quantization~\citep{han2016deep}, subspace training~\citep{li2018measuring}, low-rank approximation, variational inference~\citep{louizos2017bayesian,hinton1993keeping}, or some combination~\citep{lotfi2022pac,lotfi2024non,demoss2025complexity}. These methods have highlighted that neural networks can often be highly compressed, i.e. their true complexity can be much smaller than naive parameter counting would suggest~\citep{blier2018description}. While such methods have thus far not led directly to MDL-inspired regularizers that reliably improve generalization, these methods have been successful for establishing tighter compression bounds~\citep{lotfi2022pac,lotfi2024non} and inspiring parameter-efficient fine-tuning methods such as LoRA ~\citep{hu2022lora}.

However, key theoretical questions remain. These methods from prior work can be interpreted as defining different description length measures over a network's weights. Different measures may capture different types of regularities, but fail to capture \emph{all} potential regularities in the network. MDL objectives based on such measures may therefore fail to incentivize capturing all of the regularities in the data, if certain patterns cannot be encoded efficiently in the network's weights. This can lead to sub-optimal compression and---per the MDL principle---sub-optimal generalization. Ideally, we want a description length objective that encourages the model to capture \emph{any} regularity in the data. Assuming perfect optimization, such an objective would lead to optimal compression, regardless of the dataset. To what extent can we implement such an idealized objective?

Algorithmic information theory offers a framework to address this question~\citep{li2008introduction,hutter2005universal,schmidhuber1997discovering}. MDL was inspired by Solomonoff's theory of inductive inference~\citep{solomonoff1964formal}, which proposed to favor hypotheses with low Kolmogorov complexity. Put simply, the Kolmogorov complexity of an object is the length of the shortest program that generates that object (see Section~\ref{sec:background}). Kolmogorov complexity offers optimal compression relative to any other computable description length measure, up to an additive constant. This \emph{universality} is rooted in the Church-Turing thesis that any sufficiently powerful model of computation can simulate any other. Computable resource-bounded approximations maintain this universality in the limit as resource bounds increase. The notion of Kolmogorov complexity can be directly applied to program induction -- priors based on program length have been empirically successful for inducing programs~\cite[e.g.,][]{romera2024mathematical} and grammars~\cite[e.g.,][]{shaw2021compositional} with strong generalization. However, applying this notion to neural networks is less straightforward. While we can easily compute the length of a discrete program, it is less clear how to quantify the complexity of the function a neural network computes in an analogous way. Therefore, a conceptual gap remains between theoretical frameworks based on algorithmic information theory and practical description length objectives for neural networks. We aim to narrow this gap. 

Building on the theory of Kolmogorov complexity and the computational universality of Transformers, we demonstrate the existence of \emph{asymptotically optimal} description length objectives for Transformers, with optimality guarantees analogous to those that hold for Kolmogorov complexity. Our theoretical framework therefore highlights a potential path forward for identifying description length objectives for neural networks that select for models offering greater compression -- and potentially greater generalization. Specifically, we develop the theory of asymptotically optimal description length measures for neural networks through the following contributions:
\begin{itemize}
    \item We define \emph{universal two-part codes} for probabilistic models whose minimum description length for any data sample is provably optimal, up to an additive constant, relative to \emph{any} other two-part code. %
    This framework circumvents the need for arbitrary domain-specific priors by leveraging the universal properties of Kolmogorov complexity. (Section~\ref{sec:two-part-codes})
    \item We prove the existence of \emph{asymptotically optimal families of codes for Transformers}, which are universal in the limit as resource bounds increase. This result is established via a new demonstration that Transformer encoders are computationally universal in their ability to represent any computable, rational-valued conditional probability distribution. (Section~\ref{sec:two-part-codes})
\item We further prove that \emph{tractable and differentiable} objectives for Transformers can be asymptotically optimal, establishing this by constructing and analyzing a variational objective based on an adaptive Gaussian mixture prior. (Section~\ref{sec:variational-codes})
    \item We analyze various codelength objectives \emph{empirically} and analytically. Using a manually constructed solution for an algorithmic task, we show that our variational objective  selects for a highly compressible model with strong generalization. However, we find that standard optimizers fail to discover such low-complexity solutions from a random initialization, highlighting a key challenge for future work. Additionally, we analytically derive asymptotic bounds for various alternative two-part codes. (Section~\ref{sec:experiments})

\end{itemize}

\vspace{-0.1cm}
\section{Background: Kolmogorov Complexity}
\label{sec:background}
\vspace{-0.1cm}

We give a brief introduction to Kolmogorov complexity, with an extended and more formal version in Appendix \ref{sec:app-background}. 
 Intuitively, the \emph{Kolmogorov complexity} $K(x)$ of an object $x$ is the length of the shortest program, written in some standard programming language, that prints $x$. Similarly, the Kolmogorov complexity $K(f)$ of a discrete function $f$ is the length of the shortest program that computes $f$. To formalize these notions, Kolmogorov complexity can be defined with respect to \emph{universal prefix Turing machines}, which are multi-tape Turing machines that represent programs as binary strings encoded on a read-only \emph{program tape}. The key \emph{invariance theorem} states that $K(x)$ is independent of the particular choice of universal prefix Turing machine up to an additive constant. %
 While Kolmogorov complexity is not strictly computable due to the halting problem, we write $K_{T,R}(f)$ to denote a computable approximation under a \emph{resource bound} $R = (R_t, R_s) \in \N^2$, which considers only programs that terminate within $R_t$ steps and require at most $R_s$ registers on any tape, with respect to universal prefix Turing machine $T$. As resource bounds increase, $K_{T,R}(f)$ monotonically decreases to $K(f)$, up to an additive constant, per the invariance theorem. We introduce concise notation to express such relations that hold up to additive constants next.

\vspace{-0.1cm}
\paragraph{Inequalities with additive constants} Let $f: \mathcal{X} \to \mathbb{R}$ and $g: \mathcal{X} \to \mathbb{R}$ be functions over the same domain $\gX$. Then we can write $\forall x,~f(x) \leqc g(x)$ if and only if there exists some constant $c\in \mathbb{R}$ such that $\forall x,~f(x) < g(x) + c$, where the additive constant $c$ must not depend on $x$. We write $\forall x,~f(x) \eqc g(x)$ to denote that both
$\forall x,~f(x) \leqc g(x)$ and $\forall x,~g(x) \leqc f(x)$ hold.\footnote{Extending this notation to codelength functions over variables $X$ and $Y$, $\forall X,Y,~L_1(Y \mid X) \leqc L_2(Y \mid X)$ implies $\exists c \in \mathbb{R},\forall X,Y,~L_1(Y \mid X) < L_2(Y \mid X) + c$, where $c$ does not depend on $X$ or $Y$. We omit explicit quantifiers when they are clear from context.}

\vspace{-0.1cm}
\section{Problem Setting and Model Functions}
\label{sec:problem-setting}

\paragraph{Problem setting}
Let the input space $\gX$ be a discrete and countable set, and the output space $\gY$ be a discrete and finite set. We define the space of all possible datasets as $\gD = (\gX \times \gY)^*$, the set of all finite sequences of input-output pairs. For any dataset $D = ((x_1, y_1), \ldots, (x_n, y_n)) \in \gD$, we denote the corresponding sequence of inputs as $X = (x_1, \ldots, x_n)$ and outputs as $Y = (y_1, \ldots, y_n)$. Our analysis focuses on properties of encoding schemes that hold for any such pair of sequences $(X,Y)$.

\vspace{-0.1cm}

\paragraph{Model functions}

Neural networks modeling distributions over a discrete output space are typically characterized by a \emph{model function}, $f$, which defines a conditional distribution over a discrete space $\gY$ by mapping an input $x \in \gX$ to a vector of unnormalized logits, $f(x) \in \gL$. Motivated by the finite-precision arithmetic used in neural networks, we assume these logits are rational numbers (i.e. $\gL = \mathbb{Q}^{|\gY|}$) which enables a straightforward definition for the Kolmogorov complexity of a model function, $K(f)$.\footnote{See Appendix~\ref{sec:real-valued-distributions} for a discussion of the complexities involved with real-valued functions.}
We denote the set of all such computable model functions as $\gF$. The conditional distribution for a function $f \in \gF$ is given by applying the softmax function, $\sigma$, to its output logits:
\begin{align}
p(Y \mid X;f) = \prod_{x,y \in X,Y} p(y \mid x;f) &&
p(y \mid x;f) = \sigma(f(x))_y \label{eq:softmax}
\end{align}
In the following sections, we apply the MDL principle to this setting, where we want to identify \emph{codes} that select for a probabilistic model of this form that optimally compresses a set of labels $Y$ given corresponding inputs $X$, yielding a \emph{description length objective}.

\vspace{-0.1cm}
\section{Two-Part Codes} 
\label{sec:two-part-codes}
\vspace{-0.1cm}

One way to apply the MDL principle is with two-part codes. Consider the cost for a sender (e.g., Alice) to transmit labels $Y$ to a receiver (e.g., Bob), if both are given the inputs $X$. With a \emph{two-part code}, Alice first transmits a model hypothesis, and then transmits the labels encoded given this model hypothesis (Figure~\ref{fig:two_part_code}). Specifically, we consider two-part codes that are parameterized by a probabilistic model function (Section~\ref{sec:problem-setting}), which we refer to simply as \emph{two-part codes} in this paper.

\begin{defn}[two-part code] A \emph{two-part code} $M$ is specified by a triplet $\langle \gH_M, m_M, \alpha_M \rangle$ with:%
\begin{itemize}
    \item A hypothesis space $\gH_M$ which we assume is countable, e.g., the parameter space of a particular neural network assuming parameters with some finite-precision.
    \item A mapping $m_M : \gH_M \rightarrow \gF$ from hypotheses to model functions (Section~\ref{sec:problem-setting}), e.g., the mapping defined by a particular neural network architecture.\footnote{This can be a partial function, e.g. a Turing machine with unbounded resources may never halt for some inputs.}
    \item A prior $\alpha_M(h)$ defining a distribution over model hypotheses, $\gH_M$.\footnote{Formally we allow the prior to be any lower semicomputable semimeasure (see Appendix \ref{sec:app-background}).}
\end{itemize}
\label{def:two-part-code}\end{defn}

\begin{figure}[t!]
    \centering
    \includegraphics[width=0.85\columnwidth,keepaspectratio]{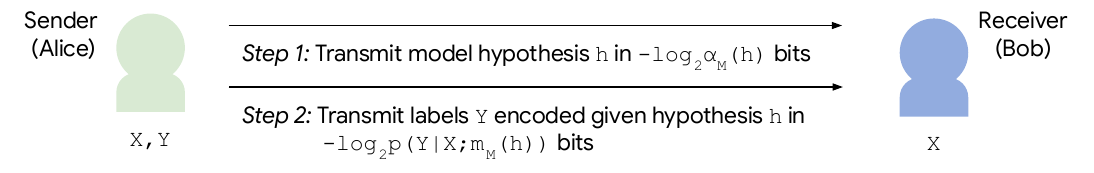}
    \caption{\textbf{Two-part code.} One way to formalize the MDL principle is with a two-part code. %
    Assume Alice and Bob agree on a two-part code $M$ and each are given inputs $X$. Alice then finds the hypothesis (e.g., model parameters) that enables sending Bob the labels $Y$ in the fewest total bits, balancing the complexity of the model with how well it fits the data. One key challenge we address is that the minimum codelength is dependent on the potentially arbitrary choice of prior $\alpha_M(h)$. %
    }
    \label{fig:two_part_code}
\end{figure}

Given a two-part code $M$, the codelength for transmitting $Y$ given $X$ is defined as the sum of the codelength of the hypothesis and the codelength of the data given the hypothesis:
\begin{align}
    L_{M}^{\text{two-part}}(Y \mid X;h) &= -\log_2 \alpha(h) -\log_2 p(Y \mid X;m_M(h)).
    \label{eq:two-part-code}
\end{align}
The $-\log_2$ terms represent the ideal codelength for transmitting the model and data under the distributions specified by the prior and the selected hypothesis, respectively; the latter term is recognizable as the standard negative log-likelihood (NLL) objective. From a Bayesian perspective, 
minimizing \eqref{eq:two-part-code} can be interpreted as finding the MAP estimate.
The minimum achievable codelength for transmitting the labels with the two-part code $M$ is then denoted as:
\begin{equation}
    L_M^{\bigast,\text{two-part}}(Y \mid X) = \min_{h \in \gH_M}
    L_M^{\text{two-part}}(Y \mid X;h).
    \label{eq:two-part-min}
\end{equation}

\paragraph{Universal two-part codes}

Notably, there is a considerable degree of freedom in specifying a two-part code, e.g. for a particular neural network architecture specifying the hypothesis space and mapping function, we can choose any prior over the model weights. The prior determines the description length measure over the weights, which, in effect, determines the inductive bias of the codelength objective.
Building on the invariance theorem for Kolmogorov complexity, we will show that there exists an equivalence class of two-part codes that are \emph{universal} in the sense that a minimizer of any universal two-part code offers at least as much compression as any other two-part code, for any dataset, up to an additive constant that does not depend on the data. This guarantee holds regardless of the specific prior, as long as it constructs a universal code, and is captured in the following theorem and its corollary.

\begin{defn}[universal two-part code]
Let $M^1$ be a two-part code. $M^1$ is a \emph{universal two-part code} if and only if, for any other two-part code, $M^2$, the following holds for any dataset $X,Y$:
\begin{equation}
L^{\bigast,\text{two-part}}_{M^1}(Y \mid X) \leqc L^{\bigast,\text{two-part}}_{M^2}(Y \mid X).
\end{equation}
\end{defn}

\begin{prop}
There exists a universal two-part code.
\label{thm:universal-two-part-existence}
\end{prop}

\begin{cor}\label{thm:universal-two-part-min} The minimum of any universal two-part code $M$ is equal to the following bound, denoted $C^{\text{two-part}}(Y \mid X)$, up to an additive constant:
\begin{equation}
L_M^{\bigast,\text{two-part}}(Y \mid X) \eqc C^{\text{two-part}}(Y \mid X) = \min_{f \in \gF} K(f) - \log_2~p(Y \mid X;f).
\end{equation}
\end{cor}
The core requirement for a two-part code to be universal is that its underlying model class must be computationally universal. That is, for any computable model function $f$, there must be some hypothesis $h$ in the hypothesis space that computes it, i.e. where $m_M(h) = f$. Furthermore, the prior, $\alpha(h)$, must assign at least one such hypothesis a codelength, $-\log_2 \alpha(h)$, that is equivalent to the Kolmogorov complexity of the function being computed, $K(f)$, up to an additive constant (see \ref{sec:app-proof-of-universal-two-part-existence} for the formal proofs). This ensures the code can, in principle, efficiently represent any computable regularity in the data. 
Fully satisfying these conditions requires a model class with unbounded computational resources, like a Turing machine, and strictly minimizing the code is not computable, due to the halting problem. Since any real-world model, such as a Transformer, has finite resources (e.g., a finite number of layers and context window), it cannot be strictly universal. This motivates the notion of \emph{asymptotically optimal} codes, introduced next. %

\paragraph{Asymptotically optimal two-part codes} 
Neural network architectures, such as Transformers, define a family of models, with hyperparameters corresponding to finite time and space resource. 
Certain families of two-part codes can then be shown to be \emph{asymptotically optimal} in the sense that as the resource bounds increase, the minimum code length monotonically decreases to the minimum of a universal two-part code (see~\ref{sec:app-proof-of-asymptotically-optimal-two-part-codes} for the proof).

\begin{defn}[asymptotically optimal families of two-part codes]
A family of two-part codes $\{ M_R \mid R \in \gR \}$ is \emph{asymptotically optimal} with respect to a universal prefix Turing machine $T$ if:
\begin{equation}\label{eq:asymptotically-optimal-two-part-code-def}
    \forall R \in \gR,~ L^{\bigast,\text{two-part}}_{M_R}(Y \mid X) \leqc C^{\text{two-part}}_{T,R}(Y \mid X),
\end{equation}
where $C^{\text{two-part}}_{T,R}(Y \mid X) := \min_{f \in \gF} K_{T,R}(f) -\log_2~p(Y \mid X;f)$ and $K_{T,R}$ denotes Kolmogorov complexity under resource bound $R$. %
\end{defn}

\begin{prop}\label{thm:asymptotically-optimal-two-part-codes}
Given an asymptotically optimal family of two-part codes $\{ M_R \mid R \in \gR \}$:
\begin{equation}
    \lim_{R_t,R_s \rightarrow \infty} L^{\bigast,\text{two-part}}_{M_R}(Y \mid X) \eqc C^{\text{two-part}}(Y \mid X),
\end{equation}
with the bound $C^{\text{two-part}}_{T,R}(Y \mid X)$ monotonically non-increasing with increasing $R_t$ or $R_s$.
\end{prop}

We demonstrate the existence of such asymptotically-optimal codes for Transformers next.

\begin{figure}[t!]
    \centering
    \includegraphics[width=0.75\columnwidth,keepaspectratio]{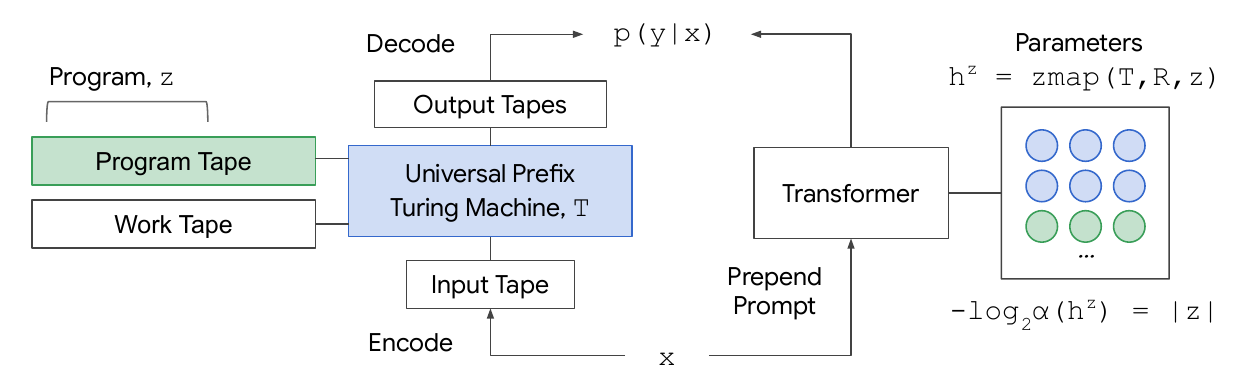}
    \caption{\textbf{Constructing asymptotically optimal codes for Transformers.} We construct a function $\zmap$ that establishes that a Transformer (right) can effectively simulate a prefix Turing machine $T$ with any prefix $z$ encoded on a program tape (left), and therefore can represent any computable model function (as defined in section~\ref{sec:problem-setting}) within some arbitrary time and space resource bound $R$. With this mapping between model functions and Transformer parameters established, we can select a prior that assigns probability to sets of parameters based on the algorithmic complexity of the function they compute, thus forming an asymptotically optimal code (Section~\ref{sec:two-part-codes-for-transformers}).}
    \label{fig:transformer_turing_emulator}
\end{figure}

\subsection{Two-part codes for Transformers}
\label{sec:two-part-codes-for-transformers}

The Transformer architecture~\citep{vaswani2017attention} defines a family of models with various hyperparameters determining the time (e.g., number of layers) and space (e.g., context window size) resources of the model. We focus our theoretical analysis on Transformer \emph{encoders}
applied to sequence classification, a commonly studied setting
\cite[e.g.,][]{devlin2019bert}.

\begin{thm}\label{thm:transformer-two-part-code}
There exists an \emph{asymptotically optimal} family of two-part codes for Transformer encoders.
\end{thm}

Figure~\ref{fig:transformer_turing_emulator} captures the main idea.
Given a universal prefix Turing machine $T$ and resource bound $R$, let $\gZ_{T,R}$ denote the set of programs such that $T$ halts within $R_t$ steps and requires only up to $R_s$ registers on any tape, for any input $x \in \gX$ encoded on the input tape. Let $f^z_T: \gX \rightarrow \gL$ be the \emph{model function} computed by $T$ with prefix $z$ on the program tape. For a given $R$, we construct a two-part code $M_R$ where the hypothesis space $\gH_{M_R}$ and mapping function $m_{M_R}$ are defined by a Transformer encoder, with $\gO(R_t)$ layers and a context window size of $\gO(R_s)$. 
We use ALTA~\citep{shaw2024alta}, a compiler from symbolic programs to Transformer weights, to construct a function $\zmap$ such that $\forall z \in \gZ_{T,R},~m_{M_R}(h^z) = f^z_T$ where $h^z = \zmap(T,R,z)$. While future work could consider alternative ways of constructing such a mapping (see \ref{sec:alternative-families}), our particular construction relies on prepending a sequence of $R_s$ ``prompt tokens'' to the model input, with learnable embeddings. These prompt tokens represent the program $z$, while the attention and MLP weights are configured to act as an interpreter, emulating the Turing machine $T$. We can then trivially construct a prior $\alpha_{M_R}(h)$ that such that $\forall z \in \gZ_{T,R},~-\log_2 \alpha_{M_R}(h^z) = |z|$, with $\alpha_{M_R}(h) = 0$ for hypotheses outside the range of $\zmap$. This construction can be shown to satisfy the conditions of an asymptotically optimal family of codes with respect to $T$. Further, any prior satisfying the more relaxed condition $\forall z \in \gZ_{T,R},~-\log_2 \alpha_{M_R}(h^z) < |z| + c_T$ where $c_T$ is a constant not depending on $R$ or $z$ is sufficient. (See \ref{sec:app-proof-transformer-two-part-code} for the formal proof.) 

\subsection{Limitations \& Practical considerations}
\label{sec:practical-considerations}

While asymptotically optimal codes offers appealing compression bounds, in this section we discuss why these bounds are insufficient to guarantee real-world performance and relevant practical considerations.

First, the theoretical guarantees  hold in the limit as model resources increase, and hold up to an additive constant that only becomes negligible as dataset complexity increases. While the community has been training increasingly large models on increasingly large datasets, any real-world models and datasets are, of course, finite. For a finite-resourced Transformer, satisfying the conditions of an asymptotically optimal code only formally constrains the description length of a subset of the computable functions that could be represented. Strictly emulating a Turing machine can be inefficient, failing to leverage a Transformer's capacity for parallel and numerical computation. Therefore, intuitively, we should consider sufficiently \emph{general and flexible priors} that are not overly specialized to a specific Turing machine. For a sufficiently flexible prior, our method used to prove asymptotic optimally -- showing that a Transformer can emulate a universal prefix Turing machine -- serves primarily to establish a worst-case upper bound on the minimum achievable codelength. 

Second, the formal guarantees apply to the \emph{minimum} of the codelength objective. Practically, we need to consider objectives that are computationally tractable and can be efficiently optimized.

Therefore, we should aim to identify codes that are asymptotically optimal (or approximately optimal) but also satisfy these additional informal but intuitive criteria, which can ultimately be assessed empirically.
As one potential path toward meeting this goal, we consider \emph{variational codes}, which we study theoretically in Section \ref{sec:variational-codes} and empirically in \ref{sec:variational-code-experiments}. We also analyze the asymptotic bounds for various alternative two-part codes in \ref{sec:asymptotic-bounds}. %

\vspace{-0.12cm}
\section{Variational Codes}
\label{sec:variational-codes}
\vspace{-0.1cm}

\begin{figure}[t!]
    \centering
    \scalebox{0.9}{
    \begin{tikzpicture}[
    axis/.style={thick, ->, >=stealth},
    code_label/.style={anchor=west, font=\small},
    formula_label/.style={below=0.2cm, text width=3.5cm, text centered, font=\small},
    blue_code/.style={blue, thick},
    green_code/.style={green!60!black, thick},
    v_helper_line/.style={thin, dotted, gray},
    h_helper_line/.style={thin, dotted, gray}
]

    \def\xKYX{0}       %
    \def\xUBC{2.5}     %
    \def\xUTPC{5.0}    %
    \def\xInfinity{7.5} %
    \def\xLabelCol{-5} %
    \def\lenVar{1.5}   %

    \def\yPrefix{1.7}      %
    \def\yTwoPart{1.3}     %
    \def\yUTPC{0.9}        %
    \def\yVar{0.5}         %
    \def\yQuasi{0.1}       %

    \coordinate (kyx) at (\xKYX,0);
    \coordinate (ubc) at (\xUBC,0);
    \coordinate (utpc) at (\xUTPC,0);
    \coordinate (infinity) at (\xInfinity,0);

    \draw[v_helper_line] (ubc) -- (ubc |- 0, \yPrefix);
    \draw[v_helper_line] (kyx) -- (kyx |- 0, \yPrefix);
    \draw[v_helper_line] (utpc) -- (utpc |- 0, \yPrefix);
    
    \node[formula_label] at (kyx) {$K(Y|X)$};
    \node[formula_label] at (ubc) 
    {$C^{\text{bayes}}(Y|X)$};
    \node[formula_label] at (utpc)
        {$C^{\text{two-part}}(Y|X)$};

    \node[code_label] (prefix_label) at (\xLabelCol, \yPrefix) {Computable prefix codes};
    \draw[blue_code, |->] (kyx |- 0, \yPrefix) -- (infinity |- 0, \yPrefix);
    \draw[h_helper_line] (prefix_label.east) -- (kyx |- 0, \yPrefix);

    \node[code_label] (twopart_label) at (\xLabelCol, \yTwoPart) {Two-part codes};
    \draw[blue_code, |->] (utpc |- 0, \yTwoPart) -- (infinity |- 0, \yTwoPart);
    \draw[h_helper_line] (twopart_label.east) -- (utpc |- 0, \yTwoPart);

    \node[code_label] (utpc_label) at (\xLabelCol, \yUTPC) {Universal two-part codes};
    \fill[green!60!black] (utpc |- 0, \yUTPC) circle (2.5pt);
    \draw[h_helper_line] (utpc_label.east) -- (utpc |- 0, \yUTPC);

    \node[code_label] (var_label) at (\xLabelCol, \yVar) {Variational codes};
    \draw[blue_code, |->] (ubc |- 0, \yVar) -- (infinity |- 0, \yVar);
    \draw[h_helper_line] (var_label.east) -- (utpc |- 0, \yVar);

    \node[code_label] (quasi_label) at (\xLabelCol, \yQuasi) {Quasi-universal variational codes};
    \draw[green_code, |-|] (ubc |- 0, \yQuasi) -- (utpc |- 0, \yQuasi);
    \draw[h_helper_line] (quasi_label.east) -- (ubc |- 0, \yQuasi);

\end{tikzpicture}
    } %
    \caption{\textbf{Upper and lower bounds on minimum codelengths.} The figure shows the minimum number of bits required to transmit labels $Y$ given inputs $X$, for different classes of codes. The bounds hold for any dataset ($X,Y$) up to an additive constant that does not depend on the dataset.}
    \label{fig:codelength_comparisons}
\end{figure}

An alternative approach for Alice to transmit the data labels to Bob is to use a \emph{variational code}, which considers a \emph{distribution} over hypotheses, rather than selecting a single best hypothesis. We start with some definitions and formal results, and then work towards a practical implementation.

\begin{defn}[variational code]
A \emph{variational code} $M$ consists of a hypothesis space $\gH_M$, a mapping $m_M : \gH_M \rightarrow \gF$ from hypotheses to model functions, and prior $\alpha_M(h)$ over $\gH_M$, as specified in definition~\ref{def:two-part-code}. Additionally, a variational code specifies a posterior hypothesis space, $\Phi_M$, and distribution over hypotheses, $\beta_M(h;\phi)$, parameterized by $\phi \in \Phi_M$.
\end{defn}

The conditional distribution over $Y$ given $X$ is therefore specified by the posterior parameters $\phi \in \Phi_M$, and defined by marginalizing over hypotheses:
\begin{equation}
p_M(Y \mid X; \phi) = \mathbb{E}_{h \sim  \beta_M(\cdot;\phi)} ~p(Y \mid X;m_M(h)).
\end{equation}
The codelength for a variational code $M$ is then defined as:
\begin{align}
L_M^{\text{var}}(Y \mid X,\phi) &= \mathbb{E}_{h \sim \beta_M(\cdot;\phi)} \left[
-\log_2~\alpha_M(h) + \log_2~\beta_M(h;\phi) - \log_2~p(Y \mid X; m_M(h))
\right] \label{eq:variational-code} \\
 &= \text{KL}\left[ \beta_M(\cdot;\phi) \parallel \alpha_M(\cdot) \right] - \log_2~ p_M(Y \mid X; \phi), \label{eq:variational-code-kl} 
\end{align}
which can be interpreted as the cost of transmitting the labels given the ``bits back'' argument of \citet{hinton1993keeping}.
The intuition is that parameters with higher uncertainty can be transmitted with lower cost. Equation~\ref{eq:variational-code-kl} mirrors \eqref{eq:two-part-code}, with the KL term capturing the model cost.
We denote the minimum of a variational code $M$ as:
\begin{equation}
L_M^{\bigast,\text{var}}(Y \mid X) = \min_{\phi \in \Phi_M}~ L_M^{\text{var}}(Y \mid X,\phi).
\end{equation}
A variational code can be seen as approximating an idealized but intractable Bayesian code (\ref{sec:bayesian-lower-bound}). As with two-part codes, variational codes require choosing a seemingly arbitrary prior.
We address this by defining \emph{quasi-universal variational codes}, which are analogous to universal two-part codes.

\begin{defn}[quasi-universal variational code]
A variational code $M$ is a \emph{quasi-universal variational code} if and only if:
\begin{equation}
L^{\bigast,\text{var}}_{M}(Y \mid X) \leqc C^{\text{two-part}}(Y \mid X).
\end{equation}
\end{defn}

We use the term \emph{quasi-universal} because such codes are not necessarily strictly universal among the set of variational codes. However, their minima are tightly bound per the following theorem.

\begin{prop}\label{thm:codelength-relations}
For any quasi-universal variational code $M$,
\begin{equation}
    K(Y \mid X) \leqc C^{\text{bayes}}(Y \mid X) \leqc L^{\bigast,\text{var}}_{M}(Y \mid X) \leqc C^{\text{two-part}}(Y \mid X),
\end{equation}
where $C^{\text{bayes}}(Y \mid X) := -\log_2 \sum_{f \in \gF} 2^{-K(f)} p(Y \mid X;f)$.
\end{prop}

The result is visualized in Figure~\ref{fig:codelength_comparisons}. The proof involves deriving $C^{\text{bayes}}(Y \mid X)$ as the minimum of a universal Bayesian code (\ref{sec:proof-of-thm-codelength-relations}). Under reasonable conditions, the differences between each of these quantities are expected to be small. See \ref{sec:relations-between-codes} for formal analysis and discussion.
The definition of asymptotically optimal families of two-part codes directly extends to variational codes (\ref{sec:asymptotically-optimal-var-codes}). %
Additionally, variational codes can be seen as a generalization of two-part codes (\ref{sec:var-two-part-equivalence}), and therefore our previous result constructing asymptotically optimal two-part codes for Transformers directly extends to prove the following proposition (\ref{sec:proof-of-exists-asymptotically-optimal-code}).

\begin{prop}\label{thm:exists-asymptotically-optimal-code}
There exists an asymptotically optimal family of variational codes for Transformer encoders. 
\end{prop}

Finally, we note that we can also generalize variational codes by using an \emph{adaptive prior}~\citep{hinton1993keeping}, where the prior is itself parameterized. We term such codes \emph{adaptive variational codes}, defined formally in~\ref{sec:adaptive-prior}. Adaptive priors help generalize our previously constructed codes to form a more flexible implementation, discussed next.

\subsection{Differentiable Codes via Gaussian Mixture Models}
\label{sec:transformer-practical-code}

Per the discussion in~\ref{sec:practical-considerations}, we aim to identify families of asymptotically optimal codes for Transformers that are based on flexible and general priors -- not explicitly constructed around a specific Turing machine emulation -- and are amenable to standard gradient-based optimization.
As one path towards this goal, we construct a family of adaptive variational codes where the adaptive prior and posterior distributions are both parameterized by sets of independent Gaussian mixture models (GMMs; see~\ref{sec:gaussian-parameterization}). Intuitively, a GMM prior shared across a group of weights drives compression by encouraging a low-entropy clustering of weight values around the component means, akin to soft quantization~\citep{ullrich2017soft,achterhold2018variational,han2016deep}. Such a GMM-based construction is appealing for its generality, and because we can estimate the variational codelength objective using Monte-Carlo sampling and apply standard optimizers via the reparameterization trick~\citep{kingma2014auto,jang2017categorical,maddison2017concrete}. The proof of Theorem~\ref{thm:gmm-code-is-optimal} below demonstrates that this construction is also asymptotically optimal, thus establishing a worst-case asymptotic upper bound on the minimal codelength (i.e. compression) achievable with such codes. The proof also serves to illustrate a potential path towards identifying other instances of practical and asymptotically optimal codes.

\begin{thm}\label{thm:gmm-code-is-optimal} There exists an asymptotically optimal family of adaptive variational codes for Transformer encoders where the adaptive prior and posterior distributions are both specified by products of independent GMMs.
\end{thm}

The key challenge in satisfying this theorem, in contrast to our earlier existence proofs, lies in the stronger constraints imposed on the prior distribution, particularly the independence assumptions. We show that two relatively general conditions are sufficient to establish asymptotic optimality. First, we use layerwise weight sharing, as in a Universal Transformer~\citep{dehghani2019universal}, ensuring that the number of Transformer weights doesn’t scale with the time bound $R_t$. Second, we share a GMM prior over groups of Transformer weights, ensuring that the number of prior parameters doesn't scale with the space bound $R_s$.

The formal proof is in~\ref{sec:transformer-var-code-details}. Our proof sketch here builds on the function $\zmap$ and the related notation introduced in Section~\ref{sec:two-part-codes-for-transformers}. Given a universal prefix Turing machine $T$, for every resource bound $R \in \gR$ and program $z \in \gZ_{T,R}$, we construct GMM posterior parameters $\phi^{R,z}$ and GMM prior parameters $\psi^{R}$ such that the KL divergence is equal to $|z|$ plus an additive constant and the model function is deterministic and equivalent to $f_T^z$. Specifically, we construct $\phi^{R,z}$ to approximate a delta function at each weight in $\zmap(T,R,z)$, except for weights representing the program tape past the $|z|$th register; these weights do not affect the model function being computed, and we therefore set their posterior to be equal to their respective prior, such that they can be transmitted ``for free''.
Importantly, by construction, the output of $\zmap$ consists of a constant number of weights (depending on $T$ but not $R$) outside of the $R_s$ prompt embeddings (which represent the program tape contents). By sharing a GMM prior across each feature column of the prompt embeddings table, the total KL for these embeddings is shown to be $|z|$ bits.
The fixed number of remaining weights can be transmitted in a constant number of bits for any GMM prior assigning non-zero probability to the associated weights in $\zmap(T,R,z)$. This establishes the necessary upper bound on the minimum codelength, although there may be prior and posterior parameters offering even greater compression.

\section{Experiments \& Analysis}
\label{sec:experiments}


\subsection{Empirical Analysis of Variational Code}
\label{sec:variational-code-experiments}

First, we empirically evaluate the GMM-based adaptive variational code discussed in Section~\ref{sec:transformer-practical-code} as a computationally tractable and differentiable instance of an asymptotically optimal code for Transformers.

\paragraph{Computing parity} We focus our initial investigation on the task of computing parity, i.e. given a sequence of 0s and 1s, determining whether there is an odd or even number of 1s. Computing parity is a common task for assessing Transformer expressivity and generalization~\citep{hahn2020theoretical,bhattamishra2020ability,chiang2022overcoming,ruoss2023randomized,deletang2022neural,anil2022exploring,zhou2022teaching,zhou2023algorithms}. Standard Transformers trained with maximum likelihood objectives struggle with consistent length generalization on this task, making it a useful benchmark for assessing whether an alternative objective can encourage stronger out-of-distribution generalization. 
A key objective of our experiments is to distinguish between (1) failures of the proposed objective to select for compressible models with strong generalization and (2) failures to effectively optimize the proposed objective. To this end, for simple algorithmic tasks such as computing parity, we can again use ALTA, this time to manually determine a set of parameters that fits the training data and also has low complexity according to the proposed objective, establishing an upper bound on the minimum of the variational objective. We can then determine whether our optimization process, starting from a random initialization, can find a set of parameters with similar or lower loss. See~\ref{sec:transformer-exp-details} for details.

The results for Transformers with different training objectives and initializations are shown in Table~\ref{tab:transformer_results}. Only the manually initialized model exhibits strong length generalization (OOD accuracy). All randomly initialized models fit the training set, but struggle with length generalization. The model trained with the variational objective from a random initialization does not achieve a loss comparable to the ALTA-compiler initialization. This suggests that while the objective may select for models with stronger generalization, standard optimization techniques fail to effectively find a strong minimizer, which we investigate further next.

\begin{table}[t!]
\centering
\caption{Transformer performance on parity task for different training objectives and initializations.}
\vspace{0.15cm}
\scalebox{0.85}{
\begin{tabular}{llccccc}
\toprule
\textbf{Init.} & \textbf{Objective} & \textbf{KL} (bits) & \textbf{Train NLL} (bits) &  \textbf{Train Acc.} & \textbf{OOD Acc.} \\
\midrule
        Random & MLE Baseline & --- & $2.5 \times 10^{1}$ & $100\%$ & $56.4\%$ \\
        Random & Variational & $2.8 \times 10^{6}$ & $2.4 \times 10^{2}$ & $100\%$ & $60.4\%$ \\
        Manual & Variational & $8.7 \times 10^{2}$ & $0.0$ & $100\%$ & $100\%$ \\
\bottomrule
\end{tabular}
}
\label{tab:transformer_results}
\end{table}

\paragraph{Optimization analysis} To understand why our Transformer models underfit the proposed objective when randomly initialized, we study the simplified setting of a 2-layer MLP on a simple identity task, with details in \ref{sec:mlp-experiments}. We use a shared GMM prior and Gaussian posteriors for each weight. 
Again, we see that random initialization fails to find a loss comparable to that achieved via manual initialization, indicating poor optimization of the proposed objective. By inspecting the learned prior and posterior distributions compared to those from manual initialization, we identify that one potential culprit is that the prior distribution collapses to a unimodal distribution (see Figure~\ref{fig:mlp_plots} in \ref{sec:mlp-experiments}). This is in contrast to our manually identified solution which consists of a multimodal distribution with low variance components, which leads to a significantly lower KL divergence. 
These results suggests that future work should consider alternative optimization procedures that discourage this form of collapse, or consider alternative families of asymptotically optimal codes altogether. 

\subsection{Asymptotic Bounds for Alternative Two-Part Codes}
\label{sec:asymptotic-bounds}

Next, we analyze asymptotic bounds for alternative two-part codes that may be of practical interest but do not strictly satisfy the conditions of asymptotic optimality. Recall from \ref{sec:two-part-codes-for-transformers} that the key requirement for satisfying the conditions of asymptotic optimality is that $\forall R \in \gR$ and $\forall z \in \gZ_{T,R}$, $-\log_2 \alpha_{M_R}(h^z) < |z| + c_T$ where $h^z = \zmap(T,R,z)$ and $c_T$ is an additive constant not depending on $R$ or $z$. 
In Table~\ref{tab:asymptotic-bounds}, the \emph{Description Length Bound} refers to an upper bound on this model cost $-\log_2 \alpha_{M_R}(h^z)$, up to an additive constant, for various two-part codes that do not strictly satisfy this theoretical ideal. A sub-optimal bound for this quantity leads to a sub-optimal bound for the overall minimum codelength, 
but close approximations could still be of practical interest.

 \begin{table}[h!]
\centering
\caption{Combining standard methods for quantization, adaptive selection of the prefix length, and layerwise weight sharing yields a two-part code with a description length upper bound (described in \ref{sec:asymptotic-bounds}) close to the theoretical ideal of $|z|$.  Ablating these methods degrades the bound, shifting the encoding cost away from the algorithmic complexity of the function computed by the network ($|z|$) and causing the encoding cost to scale linearly with the total parameter count of the network, driven by its maximum space ($R_s$) and time ($R_t$) resources.}
\vspace{0.15cm}
\scalebox{0.85}{
\begin{tabular}{cccc}
\toprule
\bf Quantization & \bf Adaptive Prefix Length & \bf Layerwise Weight Sharing &
\bf Description Length Bound \\
\midrule
\greencheck & \greencheck & \greencheck & $|z| + \log R_s$  \\
\redcross & \greencheck & \greencheck & $\gO(|z|) + \log R_s$ \\
\redcross & \redcross & \greencheck & $\gO(R_s)$ \\
\redcross & \redcross & \redcross & $\gO(R_s) + \gO(R_t)$ \\
\bottomrule
\end{tabular}
} 
\label{tab:asymptotic-bounds}
\end{table}

We show that combining a form of quantization, adaptive selection of the prefix length, and layerwise weight sharing yields a description length bound that is \emph{close} to the theoretically ideal description length bound of $|z|$, differing by only a $\log R_s$ term. \emph{Adaptive prefix length} means dynamically selecting the optimal prefix length, e.g. via a hyperparameter sweep; transmitting this quantity introduces the $\log R_s$ term. Sufficient quantization can be achieved by using an adaptive GMM prior similarly to the previously studied variational code, using the adaptive quantization method proposed by \citet{han2016deep}, or by restricting the prefix embeddings to a fixed vocabulary as in discrete prompt optimization. The bound becomes progressively worse as these aspects are ablated, ultimately resulting in a uniform encoding of the model weights. Note that by definition $|z| \leq R_s$, and we adopt $\gO(\cdot)$ notation to indicate a multiplicative factor $>1$. Overall, the $|z| + \log R_s$ bound can perhaps be seen as a relatively positive theoretical result for methods related to discrete prompt optimization. (See \ref{sec:quasi-optimal} for additional details and discussion.)

\subsection{Alternative Transformer Families}
\label{sec:alternative-families}

The analysis in this paper has focused primarily on scaling the time and space resource bounds of a Transformer encoder by scaling the number of layers and the length of a prefix prepended the input, respectively, as described in \ref{sec:two-part-codes-for-transformers}. Here we briefly discuss alternative families of Transformers.

\paragraph{Scaling Transformer Dimensionality} Rather than scaling the number of prefix token embeddings to represent increasingly complex programs, we could scale the dimensionality of the Transformer, effectively encoding the program tape in the MLP parameters. In \ref{sec:mlp-scaling} we discuss how a mapping analogous to $\zmap$ could be established for such a family of Transformers; however, this requires scaling the model and MLP dimensions linearly with program length. This results in the number of MLP parameters growing quadratically with program length, leading to a bound of $\gO(|z|^2) + \log R_s$ in the context of Table~\ref{tab:asymptotic-bounds} if applying similar techniques. The constructed MLP matrices are highly structured, but have rank that scales linearly with program length, and therefore cannot be easily compressed with standard techniques such as rank factorization. Therefore, identifying tighter asymptotic bounds for practical methods for MLP matrix compression would be of interest for future work. Otherwise, these results potentially indicate poor asymptotic bounds for standard methods related to MLP compression.

\paragraph{Transformer Decoders} Rather than scaling the number of layers in a Transformer encoder to scale the time resource bound, we could scale the number of intermediate decoding steps in a Transformer decoder. This would not require explicit layerwise weight sharing to maintain a constant description length with respect to a growing time resource bound. The main challenge in implementing a mapping analogous to $\zmap$ for a Transformer decoder is in representing updates to the Turing machine tape given the decoder's attention mask, but such a result could build on prior work establishing the Turing completeness of Transformer decoders under various conditions \citep{perez2021attention,merrill2024expressive,li2025constant}. Future work could also study bounds for Transformer decoders that can interact with external tools.
\section{Related Work}
\label{sec:related-work}

Our work is closely related to the notion of \emph{universal induction}~\citep{solomonoff1964formal} and related theoretical frameworks~\citep{levin1973universal,hutter2000theory,lattimore2013no,achille2021information,nakkiran2021turing}. For neural networks, a notable early work is \citet{schmidhuber1997discovering}, which introduced a probabilistic search algorithm for discovering neural networks with low Kolmogorov complexity. However, none of these theoretically compelling approaches directly lead to practical and scalable training objectives for neural networks. On the other hand, various MDL-inspired complexity measures have been proposed and evaluated for neural networks, based on variational inference~\citep{hinton1993keeping,blundell2015weight,louizos2017bayesian,blier2018description} or other methods~\citep{li2018measuring,lotfi2022pac,lotfi2024non,demoss2025complexity,abudy2025minimum}, such as quantization and low-rank approximation.
However, prior work has not demonstrated asymptotic compression guarantees for such methods. %
Our main contribution is establishing asymptotic bounds by constructing a bridge between Transformer weights and prefix Turing machines. Our results extend prior work establishing the Turing completeness of Transformers~\citep{perez2021attention,nowak2024representational}.
We provide an extended discussion of related work in Appendix \ref{sec:extended-related-work}.

\section{Conclusion}

In this paper, we introduced a framework for analyzing asymptotic bounds for description length objectives, grounded in the MDL principle and the universality of Kolmogorov complexity. We established that there exist asymptotically optimal description length objectives for Transformers, while also highlighting potential challenges related to effectively optimizing such objectives. 
Minimizing an asymptotically optimal description length objective achieves optimal compression for any dataset, up to an additive constant, in the limit as resource bounds increase. Such guarantees are particularly compelling as the computational resources available for training models continue to increase. 
Our theoretical framework therefore highlights a potential path towards identifying description length objectives leading to greater compression, and potentially greater generalization.

\clearpage

\eat{
\section*{Reproducibility Statement} Detailed proofs and theoretical analysis are included in Appendix~\ref{sec:appendix-theory}. Details needed to reproduce the experiments are included in Appendix~\ref{sec:appendix-experiments}.

\section*{Use of Large Language Models (LLMs)}

We used LLMs to help revise the paper, e.g., to draft suggestions for improving clarity and concision. We also used LLMs to try to identify typos or other issues with our proofs and technical analysis. LLMs did not play a significant role in research ideation.
}

\section*{Acknowledgements} We thank Jonathan Berant, Kenton Lee, and Tim Genewein
for useful discussions and feedback, and the anonymous reviewers for their helpful comments.

\bibliography{main}
\bibliographystyle{iclr2026_conference}

\clearpage
\appendix

\setcounter{tocdepth}{2}
\startcontents[app] 
\section*{Appendix}
\printcontents[app] %
  {l}             %
  {1}             %
  {}              %

\section{Additional Background}
\label{sec:app-background}

We first give an overview of relevant concepts related to Kolmogorov complexity and the Minimum Description Length (MDL) principle, and refer the reader to other resources for further reading, such as \citet{li2008introduction}, a reference text for Kolmogorov complexity and related topics. Other texts that cover these topics include \citet{cover2006elements} and \citet{hutter2005universal}. Some material is also covered in tutorials \citet{grunwald2008algorithmic} and \citet{grunwald2004tutorial}. Other notable resources include \citet{barron1985logically} which discusses the Kolmogorov complexity of probability distributions, \citet{rathmanner2011philosophical} which offers a more philosophical and intuitive discussion, and the introduction of \citet{schmidhuber1997discovering} which offers a concise summary for a machine learning audience. 

\subsection{Binary Strings}
\label{sec:binary-strings}

 Let $\B$ denote the set of finite \emph{binary strings} or sequences, e.g.:
\begin{equation}
\{0,1\}^* = \{\eps, 0, 1, 00, 01, 10, 11, 000, \ldots \},
\label{eq:binary-strings}
\end{equation}
with $\eps$ denoting the empty string. We denote the \emph{length} of a string $x \in \{0,1\}^*$ as $|x|$, e.g. $|0101| = 4$.

Note that there is a one-to-one correspondence between the set of binary strings and the natural numbers $\gN$, e.g. \eqref{eq:binary-strings} shows a standard enumeration. Therefore, there is also a one-to-one mapping between the elements of any countable set and binary strings. 

\subsection{Codes and Description Lengths}
\label{sec:codes-and-description-lengths}

We consider settings where a sender (e.g., Alice) wants to transmit some information to a receiver (e.g., Bob). Formally, consider the case where Alice wants to transmit an element from some countable set $\gX$ to Bob by sending a message consisting of some binary string, such that Bob can reconstruct the element from the transmitted message. Thus, prior to communication, Alice and Bob must agree on a \emph{code}, or \emph{description method}, which can be characterized by a \emph{decoding function} $D : \{0,1\}^* \rightarrow \gX$. We refer to the binary strings in the domain of the decoding function as \emph{codewords} (or \emph{programs}, for reasons which will become clear in the following sections). For each code $D$, we can therefore associate a \emph{description length measure}, or \emph{complexity measure}, $L_D: \gX \rightarrow \gN$, where the complexity measure's value for every $x$ is equal to the length of the shortest program that generates $x$:
\begin{align}
    L_D(x) &= \text{length of the shortest  program}~z~\text{such that}~D(z)=x \nonumber \\ 
        &= \min_{z  ~:~ D(z) = x} |z| \label{eq:description-length2}
\end{align}
Therefore, a code $D$ establishes a \emph{compression scheme}, where Alice can transmit an element $x \in \gX$ to Bob using $L_D(x)$ bits.

\paragraph{Prefix codes and Kraft's inequality} If Alice wants to send Bob a sequence of elements in $\gX$, not all codes guarantee that such a sequence is \emph{uniquely decodeable}, i.e. that the sequence can be unambiguously reconstructed when the messages for each of the elements are concatenated. This ambiguity can be resolved by considering a \emph{prefix code}. A set $\gB \subset \{0,1\}^*$ is \emph{prefix-free} if no element in $\gB$ is a prefix of any other element in $\gB$. A decoding function specifies a \emph{prefix code} if its domain is \emph{prefix-free}. A common example of a prefix-free code is a \emph{Huffman code}.

\emph{Kraft's inequality} captures an important relation between description length measures and prefix codes. It states that there exists a prefix code $D$ over $\gX$ with description lengths $L_D$ if and only if:
\begin{equation}
\sum_\gX 2^{-L_D(x)} \leq 1.
\label{eq:kraft-inequality}
\end{equation}

\paragraph{Prefix codes and probability distributions} Beyond the practical advantages of self-delimiting codes, a primary theoretical motivation for considering prefix-free codes is their close connection with probability distributions. Given a probability distribution $p(x)$ over $\gX$, there exists a prefix-free code $D$ for $\gX$ such that:
\begin{equation}
L_D(x) = \lceil-\log_2~p(x)\rceil,
\end{equation}
where we round up to form integer-length codewords. One method for constructing such a code is \emph{Shannon-Fano} coding, where more likely elements are assigned shorter codewords. 
As we are primarily interested in theoretical \emph{codelengths} rather than practical \emph{codes}, we will typically omit the additive constant from this rounding and consider non-integer codelengths.

\paragraph{Universal prefix codes} Let us consider the set of all computable prefix-free partial functions mapping from binary strings to elements of some set $\gX$, denoted as $\gD$. Each decoding function $D \in \gD$ has a corresponding description length measure, $L_D$, given by equation~\ref{eq:description-length2}.
In general, $L_D$ may be an unsatisfying complexity measure for elements of $\gX$ due to its strong dependence on the potentially arbitrary choice of $D$. However, in the following section, we will show that there exists a subset of functions in $\gD$, that are \emph{universal} in the sense of offering \emph{at least as much compression} as any other description length measure, up to an additive constant that does not depend on the complexity of the element being encoded.

\subsection{Kolmogorov Complexity}
\label{sec:background-kolmogorov-complexity}

The definition of \emph{Kolmogorov complexity} (also called \emph{algorithmic complexity}) and the proof of its universality come from \citet{kolmogorov1965three}, with similar formulations appearing independently in \citet{solomonoff1964formal} and \citet{chaitin1969length}.

Intuitively, we can think of the Kolmogorov complexity $K(x)$ of an object $x$ as the shortest program written in some standard programming language (such as Python) that prints $x$. To make this notion more precise, Kolmogorov complexity is typically defined with respect to a special class of Turing machines, called prefix Turing machines. (With some effort, it could be defined with respect to any universal model of computation.) 

\paragraph{Prefix Turing machines} Prefix Turing machines are multi-tape Turing machines. As with all multi-tape Turing machines, the behavior of the machine is specified by a \emph{transition function}. At each step of computation, the transition function takes as input the current state and the register values at current ``head'' position of any readable tape. The transition function can then optionally update the register values of any writeable tapes at their current ``head'' positions and optionally move each ``head'' to an adjacent register. The machine then updates the state for the next step or halts the computation. 

The defining characteristic of a \emph{prefix} Turing machine is a binary, unidirectional, read-only \emph{program tape} with no blank symbols. The machine also consists of a bidirectional, read-write work tape. The machine also has at least one read-only input tape, and at least one write-only output tape.\footnote{Note that in some constructions, programs and inputs are encoded on the same tape, although this detail only affects the resulting definition of Kolmogorov complexity by an additive constant.}

For a prefix Turing machine $T$, we abuse notation and write $T$ to denote the partial function that the machine computes, i.e. where $y = T(x,z)$ denotes that the Turing machine $T$, with the input tape initialized with $x$ and the program tape initialized with prefix $z$, runs for some number of steps before halting with $y$ written to the output tape. The function is undefined for $(x,z)$ that do not halt.

\paragraph{Universal Turing machines} Notably, there exists a subset of prefix Turing machines that are \emph{universal} in the sense of being able to simulate the computation of any other prefix Turing machine. Such a machine can effectively take as input the description of any other Turing machine, and emulate the computation of that machine. 

\paragraph{Kolmogorov complexity} 

The more commonly used notion of Kolmogorov complexity -- used throughout this paper -- is called \emph{prefix} Kolmogorov complexity, to distinguish from \emph{plain} Kolmogorov complexity, which lacks some of the desirable formal properties.

The (prefix) Kolmogorov complexity of a string $x \in \B$ with respect to a universal prefix Turing machine $T$ is:
\begin{equation}
K_T(x) = \min_{z  ~:~ T(\epsilon,z) = x} |z|,
\end{equation}
where $\epsilon$ denotes the empty string.

Similarly, the \emph{conditional} Kolmogorov complexity of a string $y$ given $x$ is:
\begin{equation}
K_T(y|x) = \min_{z  ~:~ T(x,z) = y} |z|.
\end{equation}

Note that Kolmogorov complexity $K_T(x)$ satisfies Kraft's inequality because the set of halting programs is prefix-free.

\paragraph{Invariance theorem} 
The \emph{invariance theorem} is the key result that gives Kolmogorov complexity its \emph{universal} property. Given any two \emph{universal} prefix-free Turing machines $T^1, T^2$:
\begin{equation}
    \forall x,~| K_{T^1}(x) - K_{T^2}(x) | < c,
\end{equation}
where $c$ is a constant that depends on the choice of $T^1$ and $T^2$ but not on $x$, and therefore we write $\forall x,~K_{T^1}(x) \eqc K_{T^2}(x)$, recalling the notation from Section~\ref{sec:background}. The invariance theorem intuitively follows from the result that a \emph{universal Turing machine} can simulate the computation of any other Turing machine. Informally, the constant $c$ captures the cost of constructing an \emph{interpreter} for programs written for $T^2$ in the language of $T^1$. The result also extends to conditional Kolmogorov complexity.

Therefore, as we are primarily interested in inequalities that hold up to an additive constant, we write Kolmogorov complexity as $K(x)$ and drop the explicit dependence on a particular choice of reference machine.

\paragraph{Computability}
 Kolmogorov complexity is formally uncomputable due to the halting problem. If we try to enumerate programs from shortest to longest, some may never halt. It is, however, upper semicomputable. Any program that halts and generates the given object provides an upper bound. This upper bound becomes increasingly tight as we run more programs for more steps.

\paragraph{Universality of Kolmogorov complexity} As a corollary to the invariance theorem, Kolmogorov complexity is a \emph{universal description length measure} as described in Section~\ref{sec:codes-and-description-lengths}, i.e. for any computable prefix-free decoding function $D$ and corresponding description length measure $L_D$,
\begin{equation}
    \forall x~~K(x) \leqc L_D(x).\label{eq:kolmogrov-universality}
\end{equation}

In other words, Kolmogorov complexity can be interpreted as a \emph{compression scheme}, where Alice encodes a string $x$ by the shortest program that generates $x$. This compression scheme offers \emph{at least as much} compression of $x$, for any $x$, as any other computable prefix code up to an additive constant that does not depend on $x$. Thus, as the complexity of $x$ increases, the additive constant becomes negligible, and all such universal description length measures are asymptotically equivalent.
Conversely, any description length measure provides an upper bound on Kolmogorov complexity, up to an additive constant, however this bound may be arbitrarily loose for some objects.

\paragraph{Kolmogorov complexity of countable objects} 

The definition of Kolmogorov complexity naturally extends to countable sets other than the non-binary strings. For example, for $x \in \gX$ and $y \in \gY$ where $\gX$ and $\gY$ are countable, we have:
\begin{equation}
K(y \mid x) = \min_{z ~:~ T \left( e_{\gX}(x),z \right) = e_{\gY}(y)} |z|,
\end{equation}
or $\infty$ if no such $z$ exists, where $e_{\gX}(x) : \gX \rightarrow \B$ and $e_{\gY}(y) : \gY \rightarrow \B$ are one-to-one computable functions mapping from $x$ and $y$ to binary encodings, and $T$ is a universal prefix Turing machine.
While $K(y \mid x)$ therefore depends on a specific choice of $e_{\gX}$ and $e_{\gY}$, we leave this implicit in the notation. Since $e_{\gX}$ and $e_{\gY}$ are computable functions, the specific choice only affects $K(y \mid x)$ up to an additive constant that does not depend on $y$ or $x$, similar to the choice of $T$. 

For structured objects such as tuples, we can consider prefix Turing machines with multiple output tapes, in order to simplify the encoding of such objects. For such a universal prefix Turing machine $T$, we use the same notation as above, assuming that both $T(x,z)$ and $e_{\gY}$ output a tuple of strings with a number of elements equal to the number of output tapes. Again, such a choice only affects the resulting definition of Kolmogorov complex up to an additive constant, as the number of tapes does not affect the machine's expressive power~\citep{papadimitriou1994computational}.

\paragraph{Kolmogorov complexity of discrete functions} We can extend the definition of Kolmogorov complexity to functions. This is relatively straightforward for a function $f: \gX \rightarrow \gY$ if the domain $\gX$ and range $\gY$ are countable such that the inputs and outputs of the function can be encoded as binary strings. Again, we choose a reference universal prefix Turing machine $T$.
\begin{equation}
    K(f) = \min_{z ~:~ \forall x,~T(e_\gX(x),z) = e_\gY(f(x))} |z|,
\label{eq:kolmogorov-discrete-function}
\end{equation}
or $\infty$ if no such $z$ exists, which is the case if $f$ is uncomputable, e.g. it solves the halting problem. 
As a direct analogue to the invariance theorem for Kolmogorov complexity of strings, the same property holds for the Kolmogorov complexity of functions. 
Extensions to real-valued functions are discussed in Appendix~\ref{sec:real-valued-distributions}. 

\paragraph{Resource-bounded Kolmogorov complexity} We can consider computable, resource-bounded approximations to Kolmogorov complexity \citep[section 7]{li2008introduction}. 

Consider a \emph{resource bound} $R = (R_t, R_s) \in \gR$, representing a bound on the time $R_t \in \N$ and $R_s \in \N$ space resources available.

Given a universal prefix Turing machine $T$, let $y = T_R(x,z)$ denote the partial function computed by $T$ if the execution given $x$ and $z$ halts within $R_t$ steps and uses at most $R_s$ registers on each tape; otherwise $T_R(x,z)$ is undefined. We can then define time and space bounded Kolmogorov complexity with respect to $T$ and $R$:
\begin{equation}
    K_{T,R}(y \mid x) = \min_{z ~:~ T_R(x,z) = y} |z|,
    \label{eq:bounded-kolmogorov}
\end{equation}
or $\infty$ if no such $z$ exists, with an analogous definition for functions.

The resource-bounded version lacks the universal properties of unbounded Kolmogorov complexity. However, $K_{T,R}(y \mid x)$ is monotonically non-increasing with respect to increasing $R_t$ and $R_s$, decreasing to $K_T(y \mid x)$ as $R_t$ and $R_s$ go to $\infty$.

\paragraph{Coding theorem}
\label{sec:app-coding-theorem}

The \emph{coding theorem} is a central result in algorithmic information theory~\citep[Section 4]{li2008introduction}, and relates to the universality of Kolmogorov complexity. A key result of the coding theorem is that if $m$ is a lower semicomputable semimeasure, then:
\begin{equation}
    K(x) \leqc -\log_2 m(x).
    \label{eq:coding-theorem}
\end{equation}
where a \emph{semimeasure} is a generalization of probability distribution, such that a semimeasure $m$ over a domain $\gX$ satisfies:
\begin{equation}
\sum_{x \in \gX} m(x) \leq 1.
\end{equation}

Note that because $K(x)$ is upper semicomputable and satisfies Kraft's inequality, $2^{-K(x)}$ is a lower semicomputable semimeasure.

This result also extends to conditional Kolmogorov complexity, and, as a direct analogue, to the Kolmogorov complexity of discrete functions.

\section{Additional Theoretical Analysis and Proofs}
\label{sec:appendix-theory}

\subsection{Rationale for Rational-valued Model Functions}
\label{sec:real-valued-distributions}

Defining Kolmogorov complexity for a function $f: \gX \rightarrow \R$ with a countable domain but real-valued range is slightly more complicated than for functions with countable (e.g. rational-valued) outputs, but we can accomplish this by considering the shortest program $z$ for a given universal prefix Turing machine $T$ that can approximate $f$ to an arbitrary degree of precision $a \in N$ specified as an additional input:
\begin{equation}
    K(f) = \min \{|z| : \forall x \in \gX,a \in \N~~ \left| T(e_{(\gX \times \N)}(x,a),z) - f(x) \right| \leq 1 / a \}.
\label{eq:kolmogorov-real-valued-function}
\end{equation}

Given this definition, we could define two-part and variational codes with respect to model functions with real-valued logits, analogously to those defined with respect to model functions with rational-valued logits. However, this leads to significantly greater theoretical complexity with limited practical benefit, given that real-world neural networks represent logits with finite precision, and any real-valued function can be approximated to an arbitrary degree of precision by a rational-valued function.

\subsection{Note on ``Universal'' Terminology}
\label{sec:universal-terminology}

The term \emph{universal} is overloaded in the context of codes and compression schemes, with different usages varying with respect to the scope (the class of codes/models being compared against) and the nature of the performance guarantee. Our usage of the term in this paper follows the use of the term in the algorithmic information theory literature to describe the \emph{universality} of Kolmogorov complexity with respect to any computable prefix code, up to an additive constant. We extend this notion to slightly restricted sets of prefix codes, i.e. the specific classes of two-part, Bayesian, and variational codes discussed in this paper.

In contrast, the term \emph{universal code} also appears in the MDL literature~\citep{grunwald2004tutorial} to describe codes that are universal with respect to a specific set of candidate codes, in the sense of offering as much compression as any candidate code, up to a factor that is typically logarithmic with respect to the size of the candidate set or data sample. The term ``universal code'' is also used to describe prefix codes for integers, such as Elias codes, which offer asymptotic guarantees relative to any other prefix code assuming a monotonically decreasing probability distribution over the integers. 

\subsection{Model Functions and Notation}
\label{sec:reference-machine-and-notation}

Here we introduce standard encodings, $e_\gX$ and $e_\gL$, used to define the Kolmogorov complexity of model functions, $\gF : \gX \rightarrow \gL$, introduced in Section~\ref{sec:problem-setting}. Let us denote the class of universal prefix Turing machines compatible with these encodings as $\gT$, with the specific assumptions specified in the following paragraphs. Recall that the specific choice of universal prefix Turing machine, and encoding functions $e_\gX$ and $e_\gL$, only affect the resulting definition of Kolmogorov complexity by an additive term (\ref{sec:background-kolmogorov-complexity}), which does not affect the main inequalities we are interested in proving, which hold up to an additive term. Therefore we choose encodings to simplify the construction of a function $\zmap$ introduced later (\ref{sec:zmap-details}), which generates Transformer parameters to emulate a prefix Turing machine.

\paragraph{Input encoding} Let $\gV$ be a vocabulary of input tokens. We assume some one-to-one encoding of inputs $e_\gX : \gX \rightarrow \gV^*$ as a token sequence, which is represented on the input tape. For generality, we do not necessarily assume the input vocabulary is binary, and assume the input tape of any $T \in \gT$ has a number of symbols greater than or equal to $|\gV|$.

\paragraph{Output encoding} Recall that the output of a model function is a tuple of rational-valued logits, with one logit for each element in the output space $\gY$, i.e. $\gL = \mathbb{Q}^{|\gY|}$. We assume that any $T \in \gT$ has $|\gY|$ write-once, unidirectional output tapes, with each logit $\in \mathbb{Q}$ encoded on a separate output tape in the following format. Let $s \in \{ 0,1 \}$ denote the sign of the logit, $\texttt{numerator} \in \N_0$ be the numerator and $\texttt{denominator} \in \N$ be the denominator. The logit is then encoded on the output tape as:
\begin{equation}
s, \overbrace{1, 1, \cdots, 1}^{\texttt{numerator}}, 0, \overbrace{1, 1, \cdots, 1}^{\texttt{denominator}-1},
\end{equation}
followed by blank symbols.
Therefore, $e_\gL : \mathbb{Q}^{|\gY|} \rightarrow (\B)^{|\gY|}$.

\paragraph{Shorthand notation} We introduce the following shorthand notation related to prefix Turing machines and the encoding functions introduced above.

First, let us define a model function $f_T^z \in \gF$ computed by $T \in \gT$ with program $z$ as follows:
\begin{equation}
    f_T^z(x) = e_\gL^{-1}\left( T(e_{\gX}(x),z) \right).
\end{equation}

Next, let us define the set of programs where $T \in \gT$ halts with a valid output for any input:
\begin{equation}
    \gZ_T = \{ z : \forall x \in \gX,~T(e_\gX(x), z) \in \dom(e_\gL) \}.
\end{equation}

Additionally, let us denote the set of programs where $T \in \gT$ halts with a valid output for any input \emph{under resource bound $R$} as:
\begin{equation}
    \gZ_{T,R} = \{ z : \forall x \in \gX,~T_R(e_\gX(x), z) \in \dom(e_\gL) \}.
\end{equation}

Note that we can rewrite the definitions of Kolmogorov complexity with respect to these definitions.
\begin{equation}
    K_T(f) = \min_{z \in \gZ_T  ~:~ f_T^z = f} |z|,
\end{equation}
or $\infty$ if no such $z$ exists, and similarly: 
\begin{equation}
    K_{T,R}(f) = \min_{z \in \gZ_{T,R}  ~:~ f_T^z = f} |z|,
\end{equation}
or $\infty$ if no such $z$ exists, where $f_T^z = f$ denotes that $\forall x \in \gX,~f_T^z(x) = f(x)$.

\subsection{Analysis of Two-Part Codes}

Here we provide the proof of Proposition~\ref{thm:universal-two-part-existence}, along with supporting lemmas and related analysis of two-part codes.

\subsubsection{Lemmas~\ref{lemma-two-part-f} and ~\ref{lemma-two-part-f2}}

We introduce the following lemmas to support our later proofs.

\begin{defn}[description length of a model function]
The \emph{description length of a model function} $f$ under code $M$, denoted $L^\gF_M(f)$, is the codelength of the shortest hypothesis that maps to $f$:
\begin{equation}
    L^\gF_{M}(f) := \min_{h \in \gH_{M}  ~:~ m_M(h)=f} -\log_2 \alpha_M(h).
    \label{eq:function-description-length}
\end{equation}
\end{defn}

\begin{lemma}
\label{lemma-two-part-f}
The minimum codelength of a two-part code $M$ can be expressed as a minimization over the set of model functions realized by $M$:
\begin{equation}
    L_{M}^{\bigast,\text{two-part}}(Y|X) = \min_{f \in \im(m_M)} \left( L^\gF_M(f) - \log_2 p(Y|X;f) \right),
    \label{eq:two-part-as-function-of-f}
\end{equation}
where $\mathrm{im}(m_M)$ is the image of the mapping $m_M$, i.e., the set of all model functions $f$ such that $f=m_M(h)$ for some $h \in \gH_M$.
\end{lemma}

\begin{proof}
We start from the definition of the minimum achievable codelength, \eqref{eq:two-part-min}:
\begin{align*}
    L_M^{\bigast,\text{two-part}}(Y|X) &= \min_{h \in \gH_M} L_M^{\text{two-part}}(Y|X;h) \\
    &= \min_{h \in \gH_M} \left( -\log_2 \alpha_M(h) - \log_2 p(Y \mid X; m_M(h)) \right).
\end{align*}
We can partition the hypothesis space $\gH_M$ into disjoint sets, where each set contains all hypotheses that map to the same model function $f \in \im(m_M)$. The minimization over all $h \in \gH_M$ can then be rewritten as a two-level minimization: first, for each function $f$, we minimize over all hypotheses that produce it, and second, we minimize over all possible functions $f$:
\begin{align*}
    L_M^{\bigast,\text{two-part}}(Y|X) &= \min_{f \in \im(m_M)} \left( \min_{h \in \gH_M  ~:~ m_M(h)=f} \left( -\log_2 \alpha_M(h) - \log_2 p(Y \mid X; m_M(h)) \right) \right).
\end{align*}
For the inner minimization, all hypotheses $h$ map to the same function $f$. Therefore, the term $-\log_2 p(Y \mid X; m_M(h))$ is constant and equal to $-\log_2 p(Y \mid X; f)$. We can pull this constant term out of the inner minimization:
\begin{align*}
    L_M^{\bigast,\text{two-part}}(Y|X) &= \min_{f \in \im(m_M)} \left( \left( \min_{h \in \gH_M  ~:~ m_M(h)=f} -\log_2 \alpha_M(h) \right) - \log_2 p(Y \mid X; f) \right).
\end{align*}
By our definition in \eqref{eq:function-description-length}, the inner term is precisely the description length of the function $f$, $L^\gF_M(f)$. Substituting this back gives the final result:
\begin{align*}
    L_M^{\bigast,\text{two-part}}(Y|X) &= \min_{f \in \im(m_M)} \left( L^\gF_M(f) - \log_2 p(Y \mid X; f) \right). \qedhere
\end{align*}
\end{proof}

\begin{lemma}
\label{lemma-two-part-f2}
For any two-part code $M$ and for all $f \in \gF$:
\begin{equation}
    K(f) \leqc L^\gF_M(f)
    \label{eq:two-part-as-function-of-f2}.
\end{equation}
\end{lemma}

\begin{proof}
First, we show that $2^{-L^\gF_M(f)}$ is a semimeasure over $\gF$.

Substituting and simplifying the definition of $L^\gF_M(f)$, we have:
\begin{align}
    2^{-L^\gF_M(f)} &= \max_{h \in \gH_{M}  ~:~ m_M(h)=f} \alpha_M(h)
\end{align}
We now need to show:
\begin{align}
    \sum_{f \in \im{m_M}} \left( \max_{h \in \gH_{M}  ~:~ m_M(h)=f} \alpha_M(h) \right) \leq 1.
\end{align}
The summation is effectively over some subset of hypotheses in $\gH_M$, i.e. those that offer the shortest description length for some function. As each element of the summation is non-negative, we have the following inequality:
\begin{align*}
    \sum_{f \in \im{m_M}} \left( \max_{h \in \gH_{M}  ~:~ m_M(h)=f} \alpha_M(h) \right) \leq \sum_{f \in \im{m_M}} \left( \sum_{h \in \gH_{M}  ~:~ m_M(h)=f} \alpha_M(h) \right) = \sum_{h \in \gH_M} \alpha_M(h),
\end{align*}
and because $\alpha_M(h)$ is, by definition, a semimeasure over $\gH_M$, we have:
\begin{align*}
    \sum_{h \in \gH_M} \alpha_M(h) \leq 1,
\end{align*}
and therefore $2^{-L^\gF_M(f)}$ is a semimeasure. Similarly, $2^{-L^\gF_M(f)}$ is lower semicomputable, given that $\alpha_M(h)$ is by definition lower semicomputable.

Therefore, per the coding theorem, we have:
\begin{align*}
    K(f) &\leqc -\log_2 2^{-L^\gF_M(f)} \\
      K(f) &\leqc L^\gF_M(f). \qedhere
\end{align*} 
\end{proof}

\subsubsection{Proof of Proposition~\ref{thm:universal-two-part-existence}}
\label{sec:app-proof-of-universal-two-part-existence}

\begin{prop*}[Proposition~\ref{thm:universal-two-part-existence} restated] There exists a universal two-part code.
\end{prop*}

\begin{proof}
We construct a two-part code $U$ with respect to a universal prefix Turing machine $T \in \gT$. The components of $U$ are defined as follows:
\begin{itemize}
    \item The hypothesis space $\gH_U = \B$ is the set of binary strings, i.e. programs.
    \item The mapping function $m_U: \gH_U \to \gF$ takes a program $h \in \gH_U$ and maps it to a model function. Specifically, $m_U(h) = f_T^h$, i.e. the function computed by running the machine $T$ with program $h$.
    \item The prior $\alpha_U(h)$ is defined as $2^{-|h|}$ for $h \in \gZ_T$, where $|h|$ is the length of a halting program $h$. Otherwise, $\alpha_U(h) = 0$ if $h$ does not halt for all inputs. The prior is therefore a lower semicomputable semimeasure.
\end{itemize}

By the definition of a universal Turing machine, the image of $m_U$ is the set of all computable model functions, $\gF$.

Using Lemma~\ref{lemma-two-part-f}, we can analyze the description length of a function $f$ under our code $U$. Recall that $L_M^\gF(f)$ is the codelength of the shortest hypothesis that maps to $f$. For our code $U$, this becomes:
\begin{align}
    L^\gF_U(f) &= \min_{h \in \gH_U  ~:~ m_U(h)=f} \left( -\log_2 \alpha_U(h) \right) \nonumber \\
    &= \min_{h \in \{0,1\}^*  ~:~ m_U(h)=f} \left( -\log_2 2^{-|h|} \right) \nonumber \\
    &= \min_{h \in \{0,1\}^*  ~:~ f^h_T=f} |h|. \label{eq:lfu_is_k}
\end{align}
The final expression is, by definition, the  Kolmogorov complexity of the function $f$ with respect to machine $T$, denoted $K_T(f)$. By the invariance theorem, $L^\gF_U(f) \eqc K(f)$. By Lemma~\ref{lemma-two-part-f2}, for any other two-part code $M$, we have $K(f) \leqc L^\gF_M(f)$. Therefore, combining these results, we have a key inequality that holds for any two-part code $M$ and any function $f$:
\begin{equation}
    L_U^\gF(f) \eqc K(f) \leqc L_M^\gF(f). \label{eq:universal_description_length_bound}
\end{equation}

Finally, we show that $U$ is a universal two-part code. We must prove that for any other code $M$, $L^{\bigast,\text{two-part}}_{U}(Y \mid X) \leqc L^{\bigast,\text{two-part}}_{M}(Y \mid X)$.

Let $f_M^\ast$ be any model function that minimizes the codelength for code $M$ on a given dataset $(X,Y)$, re-written as a minimum over model functions, per Lemma~\ref{lemma-two-part-f}, such that:
\begin{equation}
    L^{\bigast,\text{two-part}}_{M}(Y \mid X) = L_M^\gF(f_M^\ast) - \log_2 p(Y|X;f_M^\ast). \label{eq:m_star_def}
\end{equation}
The minimum codelength for our universal code $U$ is, by definition, no greater than the codelength achieved by using the specific function $f_M^\ast$:
\begin{equation}
    L^{\bigast,\text{two-part}}_{U}(Y \mid X) \leq L_U^\gF(f_M^\ast) - \log_2 p(Y|X;f_M^\ast). \label{eq:u_less_than_specific}
\end{equation}
Using our key inequality from \eqref{eq:universal_description_length_bound}, we know that $L_U^\gF(f_M^\ast) \leqc L_M^\gF(f_M^\ast)$. Substituting this into the right-hand side of \eqref{eq:u_less_than_specific} yields:
\begin{equation*}
L^{\bigast,\text{two-part}}_{U}(Y \mid X) \leqc L_M^\gF(f_M^\ast) - \log_2 p(Y|X;f_M^\ast).
\end{equation*}
Now, substituting \eqref{eq:m_star_def} into the right-hand side of the expression above gives us our final result:
\begin{equation}
L^{\bigast,\text{two-part}}_{U}(Y \mid X) \leqc L^{\bigast,\text{two-part}}_{M}(Y \mid X).
\end{equation}
Since this holds for any two-part code $M$, we have shown that $U$ is a universal two-part code.
\end{proof}

\subsubsection{Discussion of Universal Two-Part Codes}
\label{sec:universal-two-part-code-discussion}

  Notably, satisfying the conditions of a universal two-part code does not require that $L_M(h) \leqc K(h)$ for all hypotheses. This is particularly notable for neural networks, where hypothesis spaces are typically highly redundant -- many different parameter sets (hypotheses) compute the same model function. Our definition allows for some of these parameter sets to be inefficiently described, with a description length far exceeding their own Kolmogorov complexity. Universality is maintained as long as for any given model function, at least one of its corresponding parameter sets is described efficiently (with description length equal to the function's Kolmogorov complexity, up to some additive constant). Therefore, we can devise universal two-part codes without necessarily needing to devise description length measures that optimally compress every possible hypothesis, a potentially significant challenge for the vast and redundant parameter spaces of neural networks.
 
 On the other hand, because some (or even most) hypotheses may be coded quite inefficiently, a universal two-part code would not necessarily be useful for post-hoc model selection across models that were not trained to optimize the given two-part description length objective.

\subsubsection{Proof of Corollary~\ref{thm:universal-two-part-min}} 

\begin{cor*}[Corollary~\ref{thm:universal-two-part-min} restated] The minimum of any universal two-part code $M$ is equal to the following bound, denoted $C^{\text{two-part}}(Y \mid X)$, up to an additive term:
\begin{equation}
C^{\text{two-part}}(Y \mid X) := \min_{f \in \gF} K(f) - \log_2~p(Y \mid X;f)
 \eqc L_M^{\bigast,\text{two-part}}(Y \mid X).
\end{equation}
\end{cor*}

\begin{proof}
Let $U$ be the specific universal two-part code constructed in the proof of Proposition~\ref{thm:universal-two-part-existence}. Its minimum codelength is given by:
\begin{equation}
    L_U^{\bigast,\text{two-part}}(Y \mid X) = \min_{f \in \gF} \left( L_U^\gF(f) - \log_2 p(Y|X;f) \right).
\end{equation}
In that proof, we established that, $L_U^\gF(f) \eqc K(f)$. Substituting this into the equation above directly shows that the minimum codelength of $U$ is equivalent to the bound $C^{\text{two-part}}(Y \mid X)$:
\begin{equation}
    L_U^{\bigast,\text{two-part}}(Y \mid X) \eqc \min_{f \in \gF} \left( K(f) - \log_2 p(Y|X;f) \right) = C^{\text{two-part}}(Y \mid X).
\end{equation}
Now, let $M$ be any other universal two-part code. By the definition of a universal two-part code, its minimum codelength must be equivalent to that of $U$, i.e., $L_{M}^{\bigast,\text{two-part}}(Y \mid X) \eqc L_U^{\bigast,\text{two-part}}(Y \mid X)$. It therefore follows that the minimum codelength for any universal code $M$ is equivalent to the bound.
\end{proof}

\subsubsection{Proof of Proposition~\ref{thm:asymptotically-optimal-two-part-codes}}
\label{sec:app-proof-of-asymptotically-optimal-two-part-codes}

\begin{prop*}[Proposition~\ref{thm:asymptotically-optimal-two-part-codes} restated] Given an asymptotically optimal family of two-part codes $\{ M_R \mid R \in \gR \}$:
\begin{equation}
    \lim_{R_t,R_s \rightarrow \infty} L^{\bigast,\text{two-part}}_{M_R}(Y \mid X) \eqc C^{\text{two-part}}(Y \mid X),
\end{equation}
with the bound $C^{\text{two-part}}_{T,R}(Y \mid X)$ monotonically non-increasing with increasing $R_t$ or $R_s$.
\end{prop*}

\begin{proof}
The proof has two parts: establishing the limit and showing monotonicity.

First, for the upper bound, the definition of an asymptotically optimal family states that for any resource bound $R$:
\begin{equation}
L^{\bigast,\text{two-part}}_{M_R}(Y \mid X) \leqc C^{\text{two-part}}_{T,R}(Y \mid X) = \min_{f \in \gF} \left( K_{T,R}(f) -\log_2 p(Y \mid X;f) \right). \end{equation}
As the resource bounds $R_t, R_s \to \infty$, the resource-bounded Kolmogorov complexity $K_{T,R}(f)$ converges to the standard Kolmogorov complexity $K_T(f)$. By the invariance theorem, $K_T(f) \eqc K(f)$, so the bound converges to the universal two-part codelength:
\begin{equation}
\lim_{R_t, R_s \to \infty} C^{\text{two-part}}_{T,R}(Y \mid X) \eqc C^{\text{two-part}}(Y \mid X).
\end{equation}
This establishes the upper bound for our limit:
\begin{equation}
\lim_{R_t, R_s \to \infty} L^{\bigast,\text{two-part}}_{M_R}(Y \mid X) \leqc C^{\text{two-part}}(Y \mid X).
\end{equation}
For the lower bound, we know from Corollary~\ref{thm:universal-two-part-min} that $C^{\text{two-part}}(Y \mid X)$ is the universal lower bound for any two-part code. Thus, for any $R$, $L^{\bigast,\text{two-part}}_{M_R}(Y \mid X) \geqc C^{\text{two-part}}(Y \mid X)$, which must also hold in the limit.

Since the limit is both upper- and lower-bounded by $C^{\text{two-part}}(Y \mid X)$ up to an additive constant, the equivalence holds.

Now we focus on showing monotonicity.
Let $R$ and $R'$ be two resource bounds such that $R'$ provides at least as many resources as $R$ (i.e., $R'_t \ge R_t$, $R'_s \ge R_s$). Any program that is computable within bounds $R$ is also computable within bounds $R'$. This implies that for any function $f$, $K_{T,R'}(f) \le K_{T,R}(f)$. Consequently, the bound on the codelength is monotonic: $C^{\text{two-part}}_{T,R'}(Y \mid X) \le C^{\text{two-part}}_{T,R}(Y \mid X)$.
\end{proof}

\subsection{Details of \zmap}
\label{sec:zmap-details}

We use the ALTA compiler~\citep{shaw2024alta} to construct (in Python) a function $\zmap(T,R,z)$ that generates Transformer weights such that the Transformer emulates the model function computed by a prefix Turing machine $T$ with program tape contents $z$ under a resource bound $R$.
Formally, we construct \zmap~to satisfy the following condition for all $x \in \gX$, $R \in \gR$, $z \in \gZ_{T,R}$, and $T \in \gT$:
\begin{equation}
m_R \left(\, \zmap(T, R, z) \,\right) = f^z_T,
\label{eq:zmap-conditions}
\end{equation}
where:
\begin{itemize}
    \item The mapping function $m_R$ prepends $R_s$ ``prompt tokens'' to the input and then runs the forward pass of the Transformer with parameters $h$, outputting the unnormalized logits prior to the final softmax. Recall that $R_s \in \N$ is a space resource bound.
    \item The output of $\zmap$ is a set of Transformer weights, with the particular model dimensions depending on $T$ and $R$.
    \item Recall $f^z_T \in \gF$ is the model function from $\gX$ to $\gL$ computed by prefix Turing machine $T \in \gT$ with program $z$, where $\gT$ is a class of universal prefix Turing machines which assumes specific encodings for $\gX$ and $\gL$ on the input and output tapes, $e_\gX$ and $e_\gL$, respectively. Also recall that $\gZ_{T,R}$ is the set of programs for $T$ that halt with a valid output under resource bounds $R$ for all inputs. (See \ref{sec:reference-machine-and-notation})

\end{itemize}

In this section we detail the construction of \zmap~and these supporting constructions, which are later used to prove the paper's key theorems.

\subsubsection{Transformer Input Preprocessing}

As described above, the mapping function $m_R$ involves prepending $R_s$ prompt tokens, and then running the forward pass of the Transformer.
Given input tokens $e_\gX(x) = (x_1, x_2, \cdots, x_{|x|}) \in \gV^*$, we prepend a sequence of $R_s$ prompt tokens $(p_1, p_2, \cdots, p_{R_s})$ from a separate vocabulary of size $R_s$. We then form the Transformer input sequence (prior to the embedding lookup) as follows:
\begin{equation}
    \texttt{START}, p_1, p_2, \cdots, p_{R_s}, \texttt{SEP}, x_1, x_2, \cdots, x_{|x|}, \texttt{END},
\end{equation}
where \texttt{START}, \texttt{SEP}, and \texttt{END} are special reserved tokens.

\subsubsection{Emulating Single-Tape Turing Machines}
\label{sec:single-tape-alta-program}

The function \zmap~generates Transformer weights that emulate a multi-tape prefix Turing machine, including decoding of the output tapes to a set of logits. Before we describe \zmap, we start with a simpler explanation of how a Transformer can emulate a single-tape Turing machine, by giving code for an ALTA program. We refer the reader to \citet{shaw2024alta} for an explanation of the ALTA language and compiler. 

\begin{figure}[ht]
\centering
\begin{minted}[
  frame=single,
  framesep=2mm, 
  fontsize=\fontsize{8}{9}\selectfont\ttfamily,
  ]{python}
class TransitionIn:
  state: int
  head_symbol: int

class Move(enum.Enum):
  LEFT = enum.auto()
  RIGHT = enum.auto()

class TransitionOut:
  state: int
  symbol: int | None
  move: Move | None
  halt: bool = False

TransitionFn = Callable[[TransitionIn], TransitionOut]

class MachineSpec:
  transition_fn: TransitionFn  # Turing machine transition function.
  num_states: int  # Turing machine number of states.
  num_symbols: int  # Number of tape symbols.
  num_steps: int  # Time resource bound.
  num_registers: int  # Space resource bound.
\end{minted}
\caption{Transition function definition for standard single-tape Turing machine.}
\label{fig:single_tape_turing_machine}
\end{figure}

\begin{figure}[ht]
\centering
\begin{minted}[
  frame=single,
  framesep=2mm, 
  fontsize=\fontsize{8}{9}\selectfont\ttfamily,
  ]{python}
from alta import program_builder as pb

def build_turing_machine_program_spec(spec: MachineSpec) -> pb.ProgramSpec:
  """Returns ALTA program spec for a Turing machine emulator."""

  variables = {
      "halted": pb.var(2),
      "state": pb.var(spec.num_states),
      "symbol": pb.input_var(spec.num_symbols),
      "head": pb.input_var(2, init_fn=lambda x: x == START),
      "one": pb.var(2, default=1),  # Constant for queries.
  }
  attention_heads = {
      # Read the symbol at the current head position.
      "head_symbol": pb.qkv("one", "head", "symbol"),
      # Identify tape cells to the left and right of the head.
      "head_left": pb.v_relative("head", -1),
      "head_right": pb.v_relative("head", 1),
  }

  def ffn_fn(x):
    """FFN function for Turing machine emulator."""
    if x["halted"]:
      return

    # Get the next action from the state transition function.
    output = spec.transition_fn(
        TransitionIn(state=x["state"], head_symbol=x["head_symbol"])
    )
    x["state"] = output.state
    if output.halt:
      x["halted"] = 1

    # Write a new symbol at current head position if specified.
    if x["head"] and output.symbol is not None:
      x["symbol"] = output.symbol

    # Move the tape head to left or right if specified.
    if output.move is not None and x["head"]:
      x["head"] = 0
    if output.move == Move.RIGHT and x["symbol"] != START and x["head_left"]:
      x["head"] = 1
    elif output.move == Move.LEFT and x["symbol"] != END and x["head_right"]:
      x["head"] = 1

  return pb.program_spec(
      ffn_fn=ffn_fn, variables=variables, heads=attention_heads,
      output_name="symbol", input_range=spec.num_symbols + 2, position_range=None,
  )
\end{minted}
\caption{ALTA program specification for emulating a single-tape Turing machine.}
\label{fig:single_tape}
\end{figure}

To start, we define data structures to represent the transition function of a single-tape Turing machine in Figure~\ref{fig:single_tape_turing_machine}. Next, in Figure~\ref{fig:single_tape}, we include a function that, given a transition function, defines an ALTA program that can be compiled to a Transformer, which emulates the provided Turing machine. The Transformer expects, as input, a sequence representing the initial state of the machine's tape.

While there are potentially many ways to emulate a Turing machine in a Transformer, representing each register at a different input position, and using relative position representations to facilitate shifting the attention head to the right or left, is a relatively straightforward implementation that efficiently leverages the Transformer's element-wise weight sharing and self-attention mechanism.

\subsubsection{Defining \zmap}

The sketch of the ALTA program for \zmap~roughly follows the program for a single-tape Turing machine from~\ref{sec:single-tape-alta-program}, extended to handle multiple tapes: the input, program, work, and output tapes.

To facilitate ``decoding'', the output tape registers are represented differently in the Transformer than the registers for other tapes. Recall the encoding of the logits on the output tapes as specified in~\ref{sec:reference-machine-and-notation}. The first bit written to the Turing machine's tape specifies the sign of the logit, and therefore for this first bit we set the value of a binary ``sign'' variable. The number of subsequent $1$s written to the Turing machine's tape specifies the numerator. For each $1$ we increment a numerical variable ``sum'' by $+1$ or $-1$ depending on the ``sign'' variable. This ``sum'' variable is only non-zero at the \texttt{START} position. After a $0$ is written to the Turing machine's tape, we transition to incrementing the denominator. This is represented by changing a binary ``key'' variable from $0$ to $1$ at subsequent positions each time the transition function specifies that a $1$ should be written (the ``key'' value is also initialized to $1$ at \texttt{START}). Finally, once the Turing machine has halted, an attention head attends to each position where ``key'' is $1$, and then averages over the ``sum'' variable, which is non-zero and equal to the numerator at exactly one position. The output of the attention head is therefore the logit value, equivalent to the value decoded from the Turing machine's corresponding output tape given the encoding function.

\begin{figure}[ht]
\centering
\begin{minted}[
  frame=single,
  framesep=2mm, 
  fontsize=\fontsize{8}{9}\selectfont\ttfamily,
  ]{python}
InputSequence = list[int]
Logits = list[float]
ModelFunction = Callable[[InputSequence], Logits]

def zmap(machine_spec: PrefixMachineSpec,
         program_prefix: list[int]) -> TransformerParameters:
  alta_program_spec = build_prefix_turing_machine_program_spec(
      machine_spec, program_prefix)
  return alta_compiler(alta_program_spec)

def mapping_fn(
    machine_spec: PrefixMachineSpec,
    params: TransformerParameters,
) -> ModelFunction:
  def model_fn(input_tokens: InputSequence) -> Logits:
    # Prepend prompt and add special tokens.
    transformer_input = preprocess_input(machine_spec, input_tokens)
    return run_transformer(params, machine_spec, transformer_input)
  return model_fn
\end{minted}
\caption{Pseudocode for \zmap~and corresponding mapping function.}
\label{fig:zmap_pseudocode}
\end{figure}

The pseudocode for both \zmap~and the corresponding mapping function are given in Figure~\ref{fig:zmap_pseudocode}. The \texttt{PrefixMachineSpec} specifies the transition function for a prefix Turing machine $T \in \gT$ as well as resource bounds $R$. We omit the full implementation of the program specification for brevity. We verified the correctness of \zmap~using unit tests.

\subsubsection{ALTA Compiler Modifications}

We make several modifications to the original ALTA compiler detailed in~\citet{shaw2024alta} to reduce the number of compiled weights, which is necessary to support our later proofs. ALTA programs represent the behavior of the MLP sub-layer as a set of \emph{MLP rules}. Each MLP rule compiles to a single dimension of the weight matrices in the MLP sub-layer. 

\begin{itemize}
    \item The original ALTA compiler represents all categorical variables as one-hot vectors. We introduce \emph{binary variables}, which function similarly to categorical variables with a range of $2$, but are represented in a single activation dimension, as either $0$ or $1$. This leads to one necessary modification to constructing the first MLP weight matrix in the case where a MLP rule needs to match against a binary variable being $0$. In this case, we include a $-1$ at the index corresponding to the binary variable. This is particularly useful to represent each bit of the program tape as a single weight.
    \item We don't require specifying fixed buckets for discretizing numerical variables. This means that numerical variables cannot be used as constraints in MLP rules, but this is not required for the \zmap~program. This also means that the number of weights in the MLP layer does not need to scale with the number of values the numerical variable can take on, which is useful given that the scale of the ``sum'' variables representing the logit numerators can grow arbitrarily.
    \item We add support for a MLP rule to increment a numerical variable by a fixed scalar. This is relatively straightforward to implement by including the fixed scalar at the index corresponding to the numerical variable in the second MLP weight matrix. This enables us to increment the ``sum'' variables by $+1$ or $-1$, with $2$ MLP rules covering both cases, regardless of the input value of the variable, by leveraging the Transformer's residual connection.
    \item We enable specifying multiple output variables for the Transformer. This way, the output projection can select the $|\gY|$ variables corresponding to the outputs of the attention heads that compute the final $|\gY|$ logit values.
\end{itemize}

\subsubsection{Compiled Weights}

The weights compiled by $\zmap$ have the following properties:
\begin{itemize}
    \item All layers share the same weights, as in a Universal Transformer~\citep{dehghani2019universal} or Looped Transformer~\citep{giannou2023looped}.
    \item The minimum dimensionality of the activations in the Transformer scales linearly with the number of input ($|\gV|$) and output ($|\gY|$) symbols and the number of Turing machine states (specified by $T \in \gT$), and does not grow with increasing resource bounds. Therefore, the number of weights in invariant to to the time resource bound $R_t$.
    \item The minimum number of attention heads similarly scales linearly with the number of output symbols, as each output tape requires 3 attention heads, and there is one output tape per output symbol. A total of 7 additional heads are needed for the program, input, and work tapes. However, the minimum number of attention heads does not scale with any resource bound.
    \item The minimum hidden dimension of the MLP layer is determined by the complexity of Turing machine's transition function, but similarly does not scale with any resource bound.
    \item The Transformer requires only relative position representations, and no absolute position representations. It requires separate bias terms only for relative positions $-1$ and $+1$, i.e. to attend to positions immediately to the left or right.
    \item The Transformer's input embedding table has $|\gV| + R_s + 3$ rows, i.e. embeddings for \texttt{START}, \texttt{SEP}, \texttt{END}, each token in the input vocabulary $\gV$, and $R_s$ prompt tokens, where $R_s$ is the space resource bound. As the embeddings for the prompt tokens therefore are the only weights that scale as a function of any resource bound, we discuss these weights specifically in the following section.
\end{itemize}

Therefore, the number of necessary weights, i.e. the hypothesis space, is a function of $T$ and $R$, but not $z$. The weight dimensions are minimums because the compiled weights can always be padded with $0$s or $-1$s. Also note that Transformers compiled by ALTA exactly compute a symbolic program in the limit as the configurable attention matrix scalar goes to $\infty$, such that the compiled Transformers implement ``hard attention''~\citep{perez2021attention}. All weight values belong to a small set of unique values, and therefore require only finite precision to represent.

\subsubsection{Prompt Embeddings}
\label{sec:prompt-embeddings}

The prompt token embeddings for prompt tokens $p_1, p_2, \cdots, p_{R_s}$ produced by \zmap~for a given program $z = z_1, z_2, \cdots, z_{|z|}$ are specified as follows:
\begin{equation}
  \begin{blockarray}{ccccc}
    \begin{block}{c[cccc]}
      p_1 & b^z_1 & w_1 & \cdots & w_{|w|}  \\
      p_2 & b^z_2 & w_1 & \cdots & w_{|w|}   \\
       \vdots & \vdots & \vdots & \ddots & \vdots   \\
      p_{R_s} & b^z_{R_s} & w_1 & \cdots & w_{|w|}   \\
    \end{block}
  \end{blockarray}
\end{equation}
where $(w_1, \cdots, w_{|w|})$ is a finite-length weight vector that does not depend on $z$ or any resource bound, and the weights $(b^z_1, \cdots, b^z_{R_s})$ encode the $|z|$ bits of $z$ in $R_s > |z|$ weights as follows:
\begin{equation}
b^z_i = \begin{cases}
1.0, & i <= |z| \land z_i = 1 \\
-1.0, & i <= |z| \land z_i = 0 \\
r, & i > |z|,
\end{cases}
\end{equation}
where $r$ can be chosen arbitrarily, as these weights represent bits on the program tape that, by construction, will never be read, because by definition the program halts after reading the final bit of $z$ and does not move the head for the unidirectional program tape any further. In the Transformer, this results in the attention head that ``reads'' the program tape never attending to these values. Therefore, we arbitrarily choose $r = 0.0$.

\subsubsection{Resource Bounds Discussion}
Since the compiled weights are the same across all layers, the resource bound $R_t$ denoting the maximum number of Turing machine steps -- and therefore the maximum number of Transformer layers -- does not affect the compilation. The Transformer requires $R_t + 2$ layers at inference time, as it executes the Turing machine's transition function once per layer, and then requires $2$ additional layers to compute the final logit values as described above.
The registers of the initially blank work and output tapes are also represented by the prompt tokens, and therefore the number of prompt tokens determines both the maximum program length and maximum number of registers on the work and output tapes.

Notably, because only the number of prompt tokens grows with increasing space resource bound $R_s$, and not any of the other aspect of the Transformer, as this bound becomes larger an increasing proportion of the Transformer's overall weights are allocated to embeddings of prompt tokens, rather than, e.g. the weights of the MLP layer.  This is quite different from how Transformer parameter counts are conventionally scaled. It's possible there are alternative ways of emulating a prefix Turing machine that would lead to different scaling behavior, e.g. by effectively representing the program within the MLP layer, but we leave consideration of this to future work.

Future work could also consider separate resource bounds for the maximum number of program tape registers (bounding the model's capacity to represent increasingly complex functions) and the number of work tape registers (bounding the model's memory capacity during the forward pass).

\subsection{Proof of Theorem~\ref{thm:transformer-two-part-code}}
\label{sec:app-proof-transformer-two-part-code}

\begin{thm*}[Theorem ~\ref{thm:transformer-two-part-code} restated]

There exists an \emph{asymptotically optimal} family of two-part codes for Transformer encoders.
\end{thm*}

\begin{proof}
Our proof builds on the construction of the function $\zmap$, detailed in \ref{sec:zmap-details}, which generates Transformer parameters satisfying \eqref{eq:zmap-conditions}.

Let $T \in \gT$ be a universal prefix Turing machine. We can define a family of two-part codes $\{ U_R \mid R \in \gR \}$ that is asymptotically optimal with respect to $T$ as follows.
\begin{itemize}
    \item The hypothesis space $\gH_{U_R}$ is the Transformer parameter space specified by the codomain of $\zmap$ given $T$ and $R$.
    \item The mapping function $m_{U_R}$ is the mapping function described in \ref{sec:zmap-details}, where $m_{U_R}(h)$ consists of prepending $R_s$ prompt tokens and running a Transformer forward pass with parameters $h$. 
    \item We can trivially define the prior as $\alpha_{U_R}(h) = 2^{-|z|}$ if there exists $z \in \gZ_{T,R}$ such that $h = \zmap(T,R,z)$, or $0$ otherwise. This prior therefore assigns non-zero probability only to those specific Transformer parameters that correspond to a valid Turing machine emulation.
\end{itemize}

Now, we must show that this family is asymptotically optimal with respect to $T$, meaning we must prove that for any resource bound $R \in \gR$ and dataset $(X,Y)$:
\begin{equation}
L^{\bigast,\text{two-part}}_{U_R}(Y \mid X) \leqc C^{\text{two-part}}_{T,R}(Y \mid X).
\end{equation}
Using Lemma~\ref{lemma-two-part-f}, we can express the codelength on the left as a minimum over functions:
\begin{equation}
 L^{\bigast,\text{two-part}}_{U_R}(Y \mid X) = \min_{f \in \mathrm{Im}(m_{U_R})} \left( L^\gF_{U_R}(f) - \log_2 p(Y|X;f) \right).
 \end{equation}
The function description length, $L^\gF_{U_R}(f)$, is the codelength of the shortest hypothesis that computes $f$. Given our prior, a hypothesis $h$ has a finite codelength only if it is the output of $\zmap$ for some program $z \in \gZ_{T,R}$. Therefore, the minimization is effectively over programs $z$:
\begin{align*}
    L^\gF_{U_R}(f) &= \min_{h \in \gH_{U_R}  ~:~ m_{U_R}(h)=f}  -\log_2 \alpha_{U_R}(h) \\
    &= \min_{z \in \gZ_{T,R}  ~:~ m_{U_R}(\zmap(T,R,z))=f} -\log_2 2^{-|z|} \\
    &= \min_{z \in \gZ_{T,R}  ~:~ f^z_T = f} |z|.
\end{align*}
The final step follows from \eqref{eq:zmap-conditions}, and this final expression is precisely the definition of the resource-bounded Kolmogorov complexity of $f$ with respect to $T$, i.e., $K_{T,R}(f)$. So, we have shown that $L^\gF_{U_R}(f) = K_{T,R}(f)$.

Substituting this equality back into the expression for the total codelength, we find:
\begin{align*}
L^{\bigast,\text{two-part}}_{U_R}(Y \mid X) &= \min_{f \in \im(m_{U_R})} \left( K_{T,R}(f) - \log_2 p(Y|X;f) \right) \\
&= C^{\text{two-part}}_{T,R}(Y \mid X).
\end{align*}
This step holds because the image of our constructed mapping function includes all functions computable by $T$ within bounds $R$.

Since we have shown that $L^{\bigast,\text{two-part}}_{U_R}(Y \mid X)$ is not just bounded by, but is \emph{equal} to $C^{\text{two-part}}_{T,R}(Y \mid X)$, the condition for being an asymptotically optimal family is satisfied.
\end{proof}

\subsection{Relations Between Codelengths}
\label{sec:relations-between-codes}

In this section we establish upper and lower bounds on the minimums of various classes of codes. To this end, we review Bayesian codes, and establish a universal lower bound on the minimums of Bayesian and variational codes. These results are used to prove Proposition~\ref{thm:codelength-relations}.

\subsubsection{Bayesian Codes}
Under a Bayesian code, labels are encoded according to their likelihood under a Bayesian posterior distribution, defined below. 
 A Bayesian code $M$, just like a two-part code, is specified by a hypothesis space $\gH_M$, a mapping $m_M$ from hypothesis to model functions, and a lower semicomputable semimeasure (i.e., prior) $\alpha_M(h)$.

The Bayesian codelength is defined with respect to the Bayesian marginal likelihood of the data:
\begin{align}
    L^{\bigast,\text{bayes}}_M(Y \mid X) &= -\log_2 p^{\text{marginal}}_M(Y \mid X) \\
        &= -\log_2 \sum_{h \in \gH_M} p(Y \mid X;m_M(h))~\alpha_M(h).
\end{align}
While the summation over all hypotheses is typically intractable, the Bayesian codelength serves as a theoretical lower bound and motivation for variational codes.

\subsubsection{Universal Bayesian Codes}

Analogously to two-part codes, we can show that there exists an equivalence class of universal Bayesian codes.

\begin{defn}[universal Bayesian code]
A Bayesian code $M^1$ is a \emph{universal Bayesian code} if and only if, for any other Bayesian code $M^2$ and for all $X,Y$:
\begin{equation}
L^{\bigast,\text{bayes}}_{M^1}(Y \mid X) \leqc L^{\bigast,\text{bayes}}_{M^2}(Y \mid X).
\end{equation}
\end{defn}

\begin{lemma}\label{thm:universal-bayesian-existence}
There exists a universal Bayesian code.
\end{lemma}

\begin{proof}
Consider a Bayesian code $U$ defined identically to the two-part code in the proof of Proposition~\ref{thm:universal-two-part-existence}, with respect to universal prefix Turing machine $T$. We will show that $U$ is a universal Bayesian code.

Similarly to Lemma~\ref{lemma-two-part-f}, we can re-write the codelength for any Bayesian code $M$ as a sum over model functions rather than hypotheses.
\begin{align}
L^{\text{bayes}}_{M}(Y \mid X) &= -\log_2 \sum_{h \in \gH_{M}} p(Y \mid X;m_{M}(h)) ~\alpha_{M}(h) \\
 &= -\log_2 \sum_{f \in \im(m_M)} p(Y \mid X;f) ~\alpha^\gF_{M}(f),
\end{align}
where:
\begin{equation}
        \alpha^\gF_M(f) = 
\sum_{h \in \gH_M  ~:~ m_M(h) = f} \alpha_M(h)
\end{equation}
or $0$ if no such $h$ exists. 
Similarly to Lemma~\ref{lemma-two-part-f2}, we have $K(f) \leqc -\log_2 \alpha^\gF_M(f)$, because $\alpha^\gF_M(f)$ is a lower semicomputable semimeasure, as by definition $\alpha_M(h)$ is a lower semicomputable semimeasure.
Additionally, for our proposed universal code, 
the quantity $\alpha^\gF_U(f)$ is recognizable as the  \emph{algorithmic probability} of $f$~\citep{li2008introduction}, and we have $K(f) \eqc -\log_2 \alpha^\gF_U(f)$. We have already shown one side of this inequality above. The other side, $-\log_2 \alpha^\gF_U(f) \leqc K(f)$ holds because by definition $\alpha^\gF_U(f)$ includes $2^{-K_T(f)}$ as an element in the summation over non-negative elements. Combining these results, we have:
\begin{equation}
    -\log_2 \alpha^\gF_U(f) \eqc K(f) \leqc -\log_2 \alpha^\gF_M(f).
\end{equation}
Therefore, there exists some positive constant $c_M$ that does not depend on $f$ such that for all $f$:
\begin{equation}
    \alpha^\gF_U(f) \geq c_M * \alpha^\gF_M(f).
\end{equation}

We need to show that the following holds for any other Bayesian code $M$:
\begin{align}
L^{\bigast,\text{bayes}}_U(Y \mid X) &\leqc L^{\bigast,\text{bayes}}_M(Y \mid X).
\end{align}

We can re-write the left side as:
\begin{align}
L^{\bigast,\text{bayes}}_U(Y \mid X) &= -\log_2 \sum_{f \in \gF} p(Y \mid X;f) \alpha^\gF_U(f) \\
        &\leqc -\log_2 \sum_{f \in \gF} p(Y \mid X;f) c_M * \alpha^\gF_M(f) \\
        &= -\log_2 c_M - \log_2 \sum_{f \in \gF} p(Y \mid X;f) \alpha^\gF_M(f) \\
        &\leqc -\log_2 \sum_{f \in \gF} p(Y \mid X;f) \alpha^\gF_M(f).
\end{align}
This expression is the definition of $L^{\bigast,\text{bayes}}_M(Y \mid X)$, and therefore we have demonstrated that $U$ is a universal Bayesian code.
\end{proof}

\begin{cor}
For any universal Bayesian code $M$:
\begin{equation}
C^{\text{bayes}}(Y \mid X) := -\log_2 \sum_{f \in \gF} 2^{-K(f)} p(Y \mid X;f) \eqc L^{\bigast,\text{bayes}}_M(Y \mid X)
\end{equation}
\end{cor}

\begin{proof}
Let $U$ be the universal Bayesian code constructed in the proof of Lemma 2. Its minimum codelength is equivalent to $C^{\text{bayes}}(Y|X)$, shown by rewriting the sum over hypotheses as a sum over functions and applying our previous result that $K(f) \eqc -\log_2 \alpha_{U}^\gF(f)$.
\begin{align*}
    L_{U}^{*,\text{bayes}}(Y \mid X) &= -\log_2 \sum_{f \in \mathcal{F}} p(Y \mid X;f) \alpha_{U}^\gF(f) \\
    &\eqc -\log_{2}\sum_{f\in\mathcal{F}}p(Y \mid X;f)2^{-K(f)} \\
    &= C^{\text{bayes}}(Y \mid X).
\end{align*}
Now, let $M$ be any other universal Bayesian code. By definition, $L_{M}^{*,\text{bayes}}(Y|X) \eqc L_{U}^{*,\text{bayes}}(Y|X)$. It directly follows that $L_{M}^{*,\text{bayes}}(Y|X) \eqc C^{\text{bayes}}(Y|X)$.
\end{proof}

\subsubsection{Universal Bayesian Codes and Universal Two-Part Codes}
\label{sec:relation-bayesian-two-part}

\begin{lemma}\label{thm:relation-bayesian-two-part}
The minimum of any universal Bayesian code is less than or equal to the minimum of any universal two-part code up to an additive term, i.e.:
\begin{equation}
C^{\text{bayes}}(Y \mid X) \leqc C^{\text{two-part}}(Y \mid X)
\end{equation}
\end{lemma}
\begin{proof}

Recall the definitions of $C^{\text{bayes}}$ and 
$C^{\text{two-part}}$:
\begin{align}
C^{\text{bayes}}(Y \mid X) &\eqc -\log_2 \sum_{f \in \gF} 2^{-K(f)} p(Y \mid X;f) \\ 
C^{\text{two-part}}(Y \mid X) &\eqc \min_{f \in \gF} K(f) - \log_2 p(Y \mid X;f) \\
&= -\log_2 \max_{f \in \gF} 2^{-K(f)}  p(Y \mid X;f).
\end{align}

We can substitute these definitions into the inequality above, and then apply $\lambda x~~2^{-x}$ to both sides, which is monotonically decreasing, and therefore we also reverse the direction of the inequality:
\begin{align}
     \sum_{f \in \gF} 2^{-K(f)} p(Y \mid X;f) \geqc \max_{f \in \gF} 2^{-K(f)} p(Y \mid X;f) 
\end{align}
The inequality holds because the left-hand side is a sum of nonnegative elements where one of the elements is the element from the right-hand side.
\end{proof}

Note that per the above inequality, it is also clear that $C^{\text{two-part}}(Y \mid X)$ will decrease towards $C^{\text{bayes}}(Y \mid X)$ as more probability converges towards any single hypothesis.

\subsubsection{Universal Bayesian Codes and Kolmogorov Complexity}\label{sec:relation-bayesian-kolmogorov}

How does $C^{\text{bayes}}(Y \mid X)$ relate to $K(Y \mid X)$? Per the coding theorem, we know that $K(Y \mid X) \leqc C^{\text{bayes}}(Y \mid X)$. Establishing a bound in the other direction is more challenging. We have defined the sets of two-part and Bayesian codes with respect a probabilistic model that makes independent predictions for each individual $(x,y)$ pair, while the definition of $K(Y \mid X)$ makes no such probabilistic independence assumptions. However, $K(Y \mid X)$ and $C^{\text{bayes}}(Y \mid X)$ may often be roughly equal for large numbers of independent and identically distributed samples~\citep[Section 5]{li2008introduction}.

\subsubsection{Bayesian Codes and Variational Codes}
\label{sec:bayesian-lower-bound}

The variational codelength defined in \eqref{eq:variational-code} has a direct relationship to the Bayesian codelength. The difference between the variational codelength and the ideal Bayesian codelength is given by the KL divergence between the variational posterior $\beta_M(\cdot;\phi)$ and the true Bayesian posterior:
\begin{align}
L_M^{\text{var}}(Y \mid X,\phi) 
&= \text{KL}\left[ \beta_M(\cdot;\phi) \parallel p^{\text{posterior}}_M(\cdot \mid Y,X) \right] - \log_2 p^{\text{marginal}}_M(Y \mid X)
\label{eq:variational-data-likelihood}
\end{align}
where the Bayesian posterior is defined as:
\begin{equation}
    p^{\text{posterior}}_M(h \mid X, Y) = \frac{P(Y \mid X;m_M(h)) ~\alpha_M(h)}{p^{\text{marginal}}_M(Y \mid X)},
\end{equation}
and the Bayesian marginal likelihood is:
\begin{equation}
    p^{\text{marginal}}_M(Y \mid X) = \sum_{h \in \gH_M} p(Y \mid X;m_M(h))~\alpha_M(h).
    \label{eq:bayes-marginal}
\end{equation}
Minimizing the variational codelength $L_M^{\text{var}}(Y \mid X,\phi)$ with respect to $\phi$ is equivalent to maximizing the Evidence Lower Bound (ELBO) from variational inference. This process drives the variational posterior $\beta_M(\cdot;\phi)$ closer to the true Bayesian posterior $P^{\text{posterior}}_M(\cdot \mid X,Y)$ (by minimizing their KL divergence), and consequently brings the variational codelength closer to the corresponding Bayesian codelength $L^{\bigast,\text{bayes}}_M(Y \mid X)$. The variational codelength can thus be seen as an efficiently computable surrogate or approximation for the ideal Bayesian codelength.

\subsubsection{Proof of Proposition~\ref{thm:codelength-relations}}\label{sec:proof-of-thm-codelength-relations}

\begin{prop*}[Proposition~\ref{thm:codelength-relations} restated]
For any quasi-universal variational code $M$,
\begin{equation}
    K(Y \mid X) \leqc C^{\text{bayes}}(Y \mid X) \leqc L^{\bigast,\text{var}}_{M}(Y \mid X) \leqc C^{\text{two-part}}(Y \mid X),
\end{equation}
where $C^{\text{bayes}}(Y \mid X) := -\log_2 \sum_{f \in \gF} 2^{-K(f)} p(Y \mid X;f)$.
\end{prop*}

\begin{proof}
This theorem combines several of our previous results.

The relationship $L_{M}^{\bigast,\text{var}}(Y \mid X) \leqc C^{\text{two-part}}(Y \mid X)$ follows directly from the definition of a quasi-universal variational code.

The relationship $C^{\text{bayes}}(Y \mid X) \leqc L_{M}^{\bigast,\text{var}}(Y \mid X)$ follows from a two-step argument. First, as established by the non-negativity of the KL divergence term in \eqref{eq:variational-data-likelihood} (discussed in Section~\ref{sec:bayesian-lower-bound}), any variational codelength is lower-bounded by its corresponding Bayesian codelength:
\begin{equation*}
    L_{M}^{\bigast,\text{bayes}}(Y \mid X) \leq L_{M}^{\bigast,\text{var}}(Y \mid X).
\end{equation*}
Second, by the existence of a universal Bayesian code (Lemma~\ref{thm:universal-bayesian-existence}), any specific Bayesian code is lower-bounded by the universal bound: 
\begin{equation*}
   C^{\text{bayes}}(Y \mid X) \leqc L_{M}^{\bigast,\text{bayes}}(Y \mid X).
\end{equation*}
Combining these two inequalities yields the desired result, with discussion relating $C^{\text{bayes}}(Y \mid X)$ and $C^{\text{two-part}}(Y \mid X)$ in \ref{sec:relation-bayesian-two-part}.

Finally, the relationship $K(Y \mid X) \leqc C^{\text{bayes}}(Y \mid X)$ follows from the discussion in \ref{sec:relation-bayesian-kolmogorov}. 
\end{proof}

\subsection{Analysis of Variational Codes}

This section provides additional definitions and analysis related to variational codes.

\subsubsection{Equivalence Between Two-Part and Variational Codes}
\label{sec:var-two-part-equivalence}

Here we show that variational codes can be seen as a generalization of two-part codes, with equivalence when we restrict the posterior hypothesis space to distributions that assign all probability to a single hypothesis.

Recall that for a two-part code we optimize the codelength over the hypothesis space directly. In a variational code, we optimize the codelength over a set of posterior parameters, which define a distribution over the hypothesis space. The codes are equivalent if we can establish a bijective mapping between the hypothesis space of a two-part code and the posterior parameter space of a variational code, such that for these corresponding values the codelengths are the same and the model functions are the same. We formalize this notion in the following definition.

\begin{defn} (equivalency for two-part and variational codes)
A variational code $M^{\text{var}}$ and two-part code $M^{\text{two-part}}$ are \emph{equivalent} if there exists a bijective mapping $u : \gH_{M^{\text{two-part}}} \rightarrow \Phi_{M^{var}}$ such that if $\phi = u(h)$ then $L^{\text{var}}_{M^{\text{var}}}(Y \mid X;\phi) = L^{\text{two-part}}_{M^{\text{two-part}}}(Y \mid X;h)$ and $p_{M^{\text{var}}}(y \mid x;\phi) = p(y \mid x;m_{M^{\text{two-part}}}(h))$.
\end{defn}

\begin{lemma}\label{thm:equiv-two-part-var}
For every two-part code $M^{\text{two-part}}$, there exists an equivalent variational code $M^{\text{var}}$.
\end{lemma}
\begin{proof}
Given $M^{\text{two-part}}$, we can define $M^{\text{var}}$ as:
\begin{align}
    \Phi_{M^{\text{var}}} &= \gH_{M^{\text{var}}} = \gH_{M^{\text{two-part}}} \\
    m_{M^{\text{var}}} &= m_{M^{\text{two-part}}} \\
    \alpha_{M^{\text{var}}}(h) &= \alpha_{M^{\text{two-part}}}(h) \\
    \beta(h;\phi) &= \llbracket \phi = h \rrbracket,
\end{align}
where $\llbracket P \rrbracket$ denotes the Iverson bracket, i.e. $1$ if $P$ is true and $0$ otherwise.

In other words, the posterior hypothesis space is restricted to distributions that assign all probability to a single hypothesis. The bijective mapping $u$ is the identity function, as $\Phi_{M^{\text{var}}} = \gH_{M^{\text{two-part}}}$.

The codelength of our variational code with $\phi = u(h) = h$ is defined as:
\begin{align*}
   L^{\text{var}}_{M^{\text{var}}}(Y|X;u(h)) &= \mathbb{E}_{h \sim \beta_{M^{\text{var}}}(\cdot;h)} \left[
-\log_2\alpha_{M^{\text{var}}}(h) + \log_2\beta_{M^{\text{var}}}(h;h) - \log_2p(Y \mid X; m_{M^{\text{var}}}(h))
\right] \\
&= -\log_2\alpha_{M^{\text{two-part}}}(h) + \log_2 \llbracket h = h \rrbracket - \log_2p(Y \mid X; m_{M^{\text{two-part}}}(h))  \\
&= -\log_2\alpha_{M^{\text{two-part}}}(h) -\log_2p(Y \mid X; m_{M^{\text{two-part}}}(h)) \\
&= L^{\text{two-part}}_{M^{\text{two-part}}}(Y|X;h).
\end{align*}

The conditional distribution for our variational code with $\phi = u(h) = h$ is given as:
\begin{align*}
    p_M(Y \mid X; \phi) &= \mathbb{E}_{h \sim  \beta_{M^{\text{var}}}(\cdot;\phi)} ~p(Y \mid X;m_{M^{\text{var}}}(h)) \\
   &= p(Y \mid X;m_{M^{\text{two-part}}}(h)).
\end{align*}

Thus $M^{\text{two-part}}$ and $M^{\text{var}}$ are equivalent given the bijective mapping $u$.
\end{proof}

\subsubsection{Asymptotically Optimal Variational codes}
\label{sec:asymptotically-optimal-var-codes}

Asymptotically optimal variational codes are defined in the same way as asymptotically optimal two-part codes.

\begin{defn}(asymptotically optimal variational code)
A family of variational codes ${\{ M_R \mid R \in \gR \}}$ is \emph{asymptotically optimal} with respect to a universal prefix Turing machine $T$ if for all $R$ and for all $X,Y$:
\begin{equation}
    L^{\bigast,\text{var}}_{M_R}(Y \mid X) \leqc \min_{f \in \gF} K_{T,R}(f) -\log_2~p(Y \mid X;f),
\end{equation}
where $K_{T,R}$ denotes resource-bounded Kolmogorov complexity.
\end{defn}

The codelength $L^{\bigast,\text{var}}_{M_R}(Y \mid X)$ monotonically decreases to be $\leqc C^{\text{two-part}}(Y \mid X)$ as the resource bounds $R$ increase.

\subsubsection{Proof of Proposition~\ref{thm:exists-asymptotically-optimal-code}}\label{sec:proof-of-exists-asymptotically-optimal-code}

\begin{prop*}[Proposition~\ref{thm:exists-asymptotically-optimal-code} restated]
There exists an asymptotically optimal family of variational codes for Transformer encoders.
\end{prop*}

\begin{proof}
This trivially follows from Theorem~\ref{thm:exists-asymptotically-optimal-code}, which establishes that there exists an asymptotically optimal family of two-part codes for Transformer encoders, and Lemma~\ref{thm:equiv-two-part-var}, which establishes that for every two-part code there exists an equivalent variational code.
\end{proof}

\subsubsection{Variational Codes with Adaptive Priors}
\label{sec:adaptive-prior}

We extend the definition of a variational code (Section~\ref{sec:variational-codes}) to include an adaptive prior, termed an \emph{adaptive} variational code.

\begin{defn}[adaptive variational code]
An \emph{adaptive variational code} $M$ consists of a hypothesis space $\gH_M$, a mapping $m_M : \gH_M \rightarrow \gF$ from hypotheses to model functions, and a prior $\alpha_M(h)$ over $\gH_M$, specified as in Definition~\ref{def:two-part-code}. Additionally, an adaptive variational code specifies:
\begin{enumerate}
    \item A \emph{prior hypothesis space} $\Psi_M$ and prior distribution over hypotheses, $\alpha_M(h;\psi)$, parameterized by $\psi \in \Psi_M$.
    \item A \emph{posterior hypothesis space} $\Phi_M$ and posterior distribution over hypotheses, $\beta_M(h;\phi)$, parameterized by $\phi \in \Phi_M$.
    \item A description length measure $L^\Psi_M$ for the prior parameters satisfying Kraft's inequality, which can be interpreted as implicitly specifying a hyperprior over the prior parameters.
\end{enumerate}
\end{defn}

Similar to a standard variational code, the conditional distribution over $Y$ given $X$ is therefore specified by the posterior parameters $\phi \in \Phi_M$, and defined by marginalizing over hypotheses:
\begin{equation}
p_M(Y \mid X; \phi) = \mathbb{E}_{h \sim  \beta_M(\cdot;\phi)} ~P(Y \mid X;m_M(h)).
\end{equation}
Different from a standard variational code, the codelength for a variational code $M$ is then defined with respect to both $\phi$ and $\psi$ as:
\begin{align}
L_M^{\text{adaptive-var}}(Y \mid X; \psi,\phi) &= L^\Psi_M(\psi) + \text{KL}\left[ \beta_M(\cdot;\phi) \parallel \alpha_M(\cdot;\psi) \right] - \log_2~ p_M(Y \mid X; \phi). \label{eq:adaptive-variational-code-kl} 
\end{align}
The codelength accounts for the cost of transmitting the prior parameters, $L^\Psi_M(\psi)$. Alice can then first send Bob these prior parameters at a cost of $L^\Psi_M(\psi)$, and the remainder of the transmission then follows that of a standard variational code, following the bits-back argument.

We similarly denote the minimum of an adaptive variational code $M$ as:
\begin{equation}
L_M^{\bigast,\text{adaptive-var}}(Y \mid X) = \min_{\psi \in \Psi_M,\phi \in \Phi_M}~ L_M^{\text{adaptive-var}}(Y \mid X;\psi,\phi).
\end{equation}
The definition of an \emph{asymptotically optimal family of adaptive variational codes} is directly analogous to the definition for standard variational and two-part codes.

Note that a two-part code could also include an adaptive prior, although we don't explicitly consider such codes in this paper.

\subsection{Analysis of Gaussian Mixture Models}

This section provides additional analysis of variational codes constructed via Gaussian Mixture Models (GMMs).

\subsubsection{Parameterization of Gaussian Mixture Models}
\label{sec:gaussian-parameterization}

In formulating variational codes, we consider Gaussian Mixture Models (GMMs) with $K$ components. We parameterize such models with parameters $\omega = \{(\mu_k, \nu_k, w_k)\}_{k=1}^K$, where for the $k$-th component, $\mu_k$ is the mean, $\nu_k$ is a parameter controlling variance, and $w_k$ is a logit for the mixture weight. The probability density function is given by:
\begin{equation}
    \text{GMM}(x; w) = \sum_{k=1}^{K} \pi_k \, \mathcal{N}(x; \mu_k, \sigma_k^2),
\end{equation}
where the \emph{mixing coefficients}, $\pi_k$, are computed via the softmax function applied to the logits $w_k$ to ensure they form a valid probability distribution (i.e., $\sum_k \pi_k = 1$):
    \begin{equation}
        \pi_k = \frac{e^{w_k}}{\sum_{j=1}^{K} e^{w_j}},
    \end{equation}
and the \emph{variances}, $\sigma_k^2$, are parameterized using the softplus function of $\nu_k$ to ensure they are strictly positive:
    \begin{equation}
        \sigma_k^2 = \log(1 + e^{\nu_k}).
    \end{equation}

A key property of this parameterization is that as $\nu_k \to -\infty$, the corresponding variance $\sigma_k^2 \to 0$. In this limit, the $k$-th Gaussian component, $\mathcal{N}(x; \mu_k, \sigma_k^2)$, converges to a Dirac delta function, $\delta(x - \mu_k)$. Consequently, this GMM formulation can approximate any discrete distribution over $K$ points by treating it as a limiting case of the mixture.

\subsubsection{Unimodal vs. Multimodal Priors}
\label{sec:unimodal-vs-multimodal-priors}

This section provides a simple, illustrative example of the benefits of a multimodal GMM prior vs. a unimodal Gaussian prior for encoding discrete information. 

Consider a setting where Alice wants to send Bob a single random bit $b \in \{0, 1\}$. Let us assume we have a simple variational code consisting of a prior and posterior over a single scalar parameter $w$. Further, assume Bob ``decodes'' the bit based on the sign of $w$. 
Assume the posterior $\beta$ is parameterized as a Gaussian, $\beta(w;\mu, \sigma^2) = \mathcal{N}(w; \mu, \sigma^2)$  with mean $\mu$ and variance $\sigma^2$.
We can consider two different parameterizations of the prior $\alpha$, as a single \emph{unimodal} Gaussian, or as a \emph{multimodal} mixture of 2 Gaussians, which we show enables a much more efficient transmission cost.
Under a variational code, Alice selects the parameters of the posterior ($\mu$ and $\sigma^2$) that minimize the cost of the transmission, which is related to the KL divergence between the prior and posterior distributions ($\text{KL}[\beta(w;\mu,\sigma^2) \parallel \alpha(w)]$) and the probability that the bit can be correctly decoded ($\mathbb{E}_{w \sim \beta(w;\mu,\sigma^2)}[\text{sgn}(w) = b]$).

\begin{figure}[ht]
    \centering
    \includegraphics[width=0.85\columnwidth,keepaspectratio]{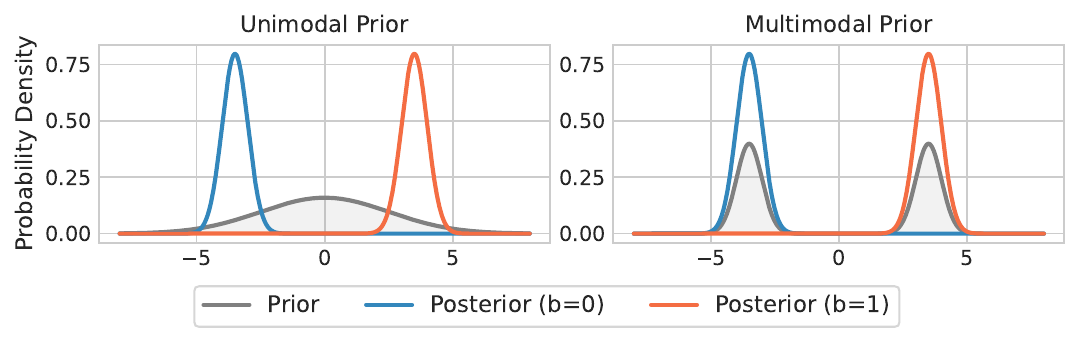}
    \caption{Visualization of posteriors for transmitting a bit $b$ under a unimodal Gaussian prior vs. multimodal GMM prior.}
    \label{fig:unimodal_vs_multimodal_prior}
\end{figure}

Figure~\ref{fig:unimodal_vs_multimodal_prior} visualizes posterior distributions for transmitting bits 0 or 1, along with a unimodal Gaussian prior (left) and a multimodal GMM prior (right). Transmission is significantly more efficient with a GMM prior, highlighting why such a prior is critical for constructing our family of asymptotically optimal variational codes. We provide more detailed analysis of each case below.

\paragraph{Unimodal (Gaussian) Prior} First, consider the case where the prior is a single zero-mean unit Gaussian, $\alpha(w) = \mathcal{N}(w \mid 0, 1)$. 
The KL divergence (in nats) can be computed analytically:
\begin{equation}
\text{KL}\left[ \beta(w;\mu,\sigma^2) \parallel \alpha(w) \right] = \frac{1}{2} \left[ \sigma^2 + \mu^2 - 1 - \log(\sigma^2) \right]
\end{equation}
To make the transmission reliable, the mean $\mu$ must be sufficiently far from $0$ and the variance $\sigma^2$ must be sufficiently small. However, increasing $\mu \gg 0$ or decreasing $\sigma^2 \ll 1$ both increase the KL divergence. Therefore, we must tradeoff the KL divergence with the probability of correct decoding, with the Pareto frontier visualized in Figure~\ref{fig:unimodal_vs_multimodal_frontier}. Reaching a probability of correct decoding close to 1 requires a KL of $\gg 1$ bits.

\begin{figure}[h!]
    \centering
    \includegraphics[width=0.5\columnwidth,keepaspectratio]{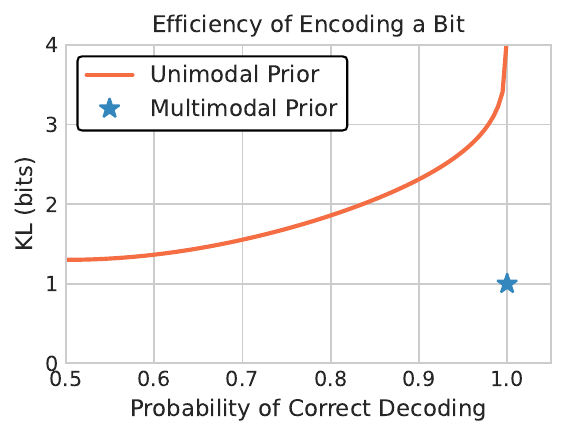}
    \caption{With a unimodal prior, we must tradeoff the KL divergence of the posterior against the expectation that the sign of the sampled parameter matches the sign of a given bit. With a multimodal prior, we can optimally encode the bit.}
    \label{fig:unimodal_vs_multimodal_frontier}
\end{figure}

\paragraph{Multimodal (GMM) Prior}
Now, let the adaptive prior be a two-component Gaussian Mixture Model (GMM) with equal weights, means $-1$ and $1$, and variance $\hat{\sigma}^2 \approx 0$.
\begin{equation*}
\alpha(w) = \frac{1}{2} \mathcal{N}(w; -1, \hat{\sigma}^2) + \frac{1}{2} \mathcal{N}(w; 1, \hat{\sigma}^2)
\end{equation*}

To send a bit $b=1$ or $b=0$, Alice can choose a posterior with mean $1$ or $-1$, respectively, and variance $\hat{\sigma}^2$.
Consider the $b=1$ case. Since $\hat{\sigma}^2 \approx 0$, the probability of correct decoding is $\approx 1$.
Additionally, the probability of any value sampled from the posterior $\beta(w;1,\hat{\sigma}^2)$ according to the prior mixture component with mean $-1$ is negligible, so the KL divergence then simplifies to be the KL divergence between the posterior and the prior component at $1$:
\begin{align*}
\text{KL}[\beta(w;1,\hat{\sigma}^2) \parallel \alpha(w)] &= \mathbb{E}_{w \sim \beta(w;1,\hat{\sigma}^2)} \left[ \log_2 \frac{\beta(w;1,\hat{\sigma}^2)}{\alpha(w)} \right] \\
&\approx \mathbb{E}_{w \sim \beta(w;1,\hat{\sigma}^2)} \left[ \log_2 \frac{\mathcal{N}(w; 1, \hat{\sigma}^2)}{\frac{1}{2} \cdot \mathcal{N}(w; 1, \hat{\sigma}^2)} \right] \\
&= \mathbb{E}_{w \sim \beta(w;1,\hat{\sigma}^2)} [\log_2 2] = 1 \text{ bit}.
\end{align*}

Thus, we can transmit the bit with an optimal KL cost of 1 bit, with near certain probability of correct decoding, as shown in Figure~\ref{fig:unimodal_vs_multimodal_frontier}. 

\subsubsection{Optimal Mixture Priors}\label{sec:optimal-mixture-priors}

In this section we discuss how adaptive GMM priors enable efficient encoding of quantized weights, as previously discussed in \citet{ullrich2017soft} and \citet{nowlan1992simplifying}.
Specifically, we will derive the optimal GMM prior parameters for the following case:
\begin{itemize}
    \item We have a group of $N$ weights with a shared, adaptive GMM prior with $K$ mixture components (\ref{sec:gaussian-parameterization}).
    \item We want to optimize the KL divergence between this prior and a posterior over these weights.
    \item The posterior for each weight approximates a delta function at a particular value.
    \item These posterior delta functions are all clustered at a small set of $M \leq K$ unique values. Let $\{m_1, m_2, \cdots, m_M\}$ denote these unique values, and let $c_i$ be the count of weights that have the value $m_i$, such that $\sum_{i=1}^M c_i = N$.
\end{itemize}
The optimal GMM prior (that minimizes KL divergence) is a mixture of $M$ components, where each component is a Dirac delta function centered at one of the unique value $m_i$. The remaining $K-M$ components should have a mixing weight of zero. Specifically, the optimal parameters, $\{(\mu_k, \nu_k, w_k)\}_{k=1}^K$, for this prior are:
\begin{itemize}
    \item \emph{Means}: The mean of each active component is set to one of the unique weight values. For $i = 1, \dots, M$, the optimal mean is $\mu_i = m_i$.
    \item \emph{Variances}: As the components converge to Dirac delta functions, their variances must approach zero. Thus, the optimal softplus parameter is $\nu_i \to -\infty$, which implies $\sigma_i^2 \to 0$.
    \item \emph{Mixing coefficients}: The mixing coefficient for each active component should be equal to the empirical frequency of the corresponding unique value in the data. For $i = 1, \dots, M$, we select $w_i$ such that the optimal mixing coefficient is $\pi_i = \frac{c_j}{N}$. The mixing coefficients for the remaining $K-M$ components are zero.
\end{itemize}

 By using a GMM with delta function components, the prior perfectly matches the empirical distribution of the weights, leading to the minimum KL divergence.
In this idealized setting, the total KL divergence is equivalent to the Shannon entropy of the empirical distribution of the weights, scaled by the total number of weights $N$. Let $p_i = c_i / N$ be the empirical frequency of the unique weight value $m_i$. The KL divergence for a single weight whose posterior is a delta function at $m_i$ simplifies to the negative log-likelihood under the prior, $-\log_2 \alpha(m_i)$. Since the optimal prior sets $\alpha(m_i) = p_i$, the cost for this single weight is $-\log_2 p_i$. Summing this cost over all $N$ weights in the group, we get the total KL divergence:
\begin{equation*}
\sum_{i=1}^{M} c_i (-\log_2 p_i) = \sum_{i=1}^{M} N p_i (-\log_2 p_i) = N \left( -\sum_{i=1}^{M} p_i \log_2 p_i \right) = N \cdot H(p)
\end{equation*}
where $H(p)$ is the Shannon entropy of the empirical distribution $p = (p_1, \cdots, p_M)$. This is the theoretical minimum number of bits required to encode the specific weight values given their empirical frequencies.

This highlights a close connection between the dynamic quantization method of \citet{han2016deep}, as employed by \citet{lotfi2022pac,lotfi2024non}, and this special case of optimizing a variational objective with an adaptive GMM prior.

\subsection{Proof of Theorem~\ref{thm:gmm-code-is-optimal}}
\label{sec:transformer-var-code-details}

\begin{thm*}[Theorem~\ref{thm:gmm-code-is-optimal} restated] There exists an asymptotically optimal family of adaptive variational codes for Transformer encoders where the adaptive prior and posterior distributions are both specified by products of independent GMMs.
\end{thm*}

\begin{proof}
We construct a family of adaptive variational codes $\{ M_R \mid R \in \gR \}$ and show it is asymptotically optimal with respect to a universal prefix Turing machine $T \in \gT$.

\paragraph{Construction of the code family}
For each resource bound $R \in \gR$, the code $M_R$ is defined as follows.
The hypothesis space $\gH_{M_R}$ consists of the weights for a Transformer encoder architecture that uses layerwise weight sharing, as specified by the construction of $\zmap(T,R,z)$ in Section~\ref{sec:zmap-details}. For a given $h \in \gH_{M_R}$, let $h_i \in h$ denote the specific Transformer weight value for some index $i$, assuming some arbitrary enumeration of the weights.
The mapping function $m_{M_R}$ is also as defined in that section, prepending $R_s$ prompt tokens to the input before the Transformer forward pass. 

The prior and posterior hypotheses spaces, $\Psi_{M_R}$ and $\Phi_{M_R}$, are both parameterized by sets of independent Gaussian mixture models (GMMs), as described in Section~\ref{sec:gaussian-parameterization}. For posterior parameters $\phi \in \Phi_{M_R}$, let $\phi_i$ denote the parameters of a GMM with index $i$. Similarly, for adaptive prior parameters $\psi \in \Psi_{M_R}$, let $\psi_{\texttt{group}(i)}$ denote the parameters of a GMM with index $\texttt{group}(i)$, where $\texttt{group}: \N \rightarrow \N$ is a function defining a grouping of Transformer weights that share the same prior GMM. 

The posterior distribution is defined as:
\begin{equation}
    \beta_{M_R}(h; \phi) = \prod_i \text{GMM}(h_i; \phi_i).
\end{equation}
The adaptive prior distribution is similarly defined as:
\begin{equation}
    \alpha_{M_R}(h; \psi) = \prod_i \text{GMM}(h_i; \psi_{\texttt{group}(i)}).
\end{equation}
Recall that the number of weights in the Transformer constructed by $\zmap$ is finite, except for the number of prompt token embeddings, where the number of rows (i.e. number of prompt tokens) grows with the space resource bound $R_s$, and therefore our asymptotic analysis largely focuses on these weights (see \ref{sec:prompt-embeddings}).

For the prompt token embeddings, a separate GMM prior is shared across each feature column.
For all weights other than the prompt token embeddings, this grouping does not formally affect our results, e.g. we can simply share a single GMM prior across each matrix or bias vector in the Transformer. 

We require that the prior and posterior GMMs corresponding to the weights of the prompt embedding table have at least $2$ components. We will show that the other GMMs only formally require a single component for the desired asymptotic bounds to hold, although a prior with more components can lead to a more efficient code in practice, as discussed below.

Finally, we construct the encoding of the prior parameters, $L_{M_R}^\Psi$, as follows. Our constructed grouping of weights sharing the same prior ensures that the total number of prior parameters in $\Psi_{M_R}$ is a small constant that does not depend on the resource bound $R$. This is because the number of parameters generated by $\zmap$ is constant with respect to $R$, other than the number of rows in the prompt token embedding table, and a single prior for each column is shared across every row of the prompt embedding table. Assuming the prior parameters can be encoded with some finite precision, we select some uniform encoding for the prior parameters, $L_{M_R}^\Psi$, which assigns non-zero probability to any specific set of prior parameters we construct below, such that they can be transmitted in $c_\Psi$ bits, which does not depend on $R$.

\paragraph{Proof of asymptotic optimality}
To prove that the family $\{M_R : R \in \gR \}$ is asymptotically optimal, we must show that for any resource bound $R \in \gR$ and any dataset $X,Y$:
\begin{equation}
L^{\bigast,\text{adaptive-var}}_{M_R}(Y \mid X) \leqc C^{\text{two-part}}_{T,R}(Y \mid X).
\end{equation}
Recall that the minimum codelength is achieved by minimizing over all possible prior and posterior parameters,
\begin{equation*}
    L^{\bigast,\text{adaptive-var}}_{M_R}(Y \mid X) = \min_{\psi \in \Psi_{M_R},\phi \in \Phi_{M_R}} L^{\text{adaptive-var}}_{M_R}(Y \mid X;\psi,\phi),
\end{equation*} 
where,
\begin{align*}
    L_{M_R}^{\text{adaptive-var}}(Y \mid X; \psi,\phi) &= L^\psi_{M_R}(\psi) + \text{KL}\left[ \beta_{M_R}(\cdot;\phi) \parallel \alpha_{M_R}(\cdot;\psi) \right] - \log_2~ p_{M_R}(Y \mid X; \phi) \\
    &= c_\Psi + \text{KL}\left[ \beta_{M_R}(\cdot;\phi) \parallel \alpha_{M_R}(\cdot;\psi) \right] - \log_2~ p_{M_R}(Y \mid X; \phi).
\end{align*}

We will show that for any resource bound $R \in \gR$ and program $z \in \gZ_{T,R}$, we can construct a specific set of prior parameters $\psi^{R}$ -- not dependent on program $z$ -- and posterior parameters $\phi^{R,z}$ such that the resulting variational codelength satisfies the following bound with respect to the function $f_T^z$ computed by prefix Turing machine $T$ with program $z$. 
That is, we will show that there exist $\psi^{R} \in \Psi_{M_R}$ and $\phi^{R,z} \in \Phi_{M_R}$ such that:
\begin{align}
L_{M_R}^{\text{adaptive-var}}(Y \mid X, \psi^R, \phi^{R,z}) \leq c_\Psi + |z| + c_T - \log_2 p(Y \mid X; f_T^z),
\end{align}
where $c_T$ and $c_\Psi$ are constants that do not depend on $X$, $Y$, $z$, or $R$.
We address the data likelihood and KL divergence terms individually next. Then we will show that this condition implies asymptotic optimality. 

Let $h^{R,z} = \zmap(T,R,z)$ be the set of weights generated by the ALTA compiler to emulate the Turing machine $T$ with program $z$.

\paragraph{Data likelihood}

For a given program $z \in \gZ_{T,R}$, we construct posterior parameters such that any weights sampled from the posterior distribution compute $f_T^z$, i.e. any sampled weights compute the same function as the weights $h^{R,z}$.

To accomplish this, we partition the Transformer weights into two disjoint subsets, which we will denote with respect a set of indexes $\gI \in \N$. For the first subset, consisting of weights with indexes $\in \gI$, we can set the parameters of the corresponding posterior GMM $\phi^{R,z}_i$ to approximate a Dirac delta function at the desired weight value, $h^{R,z}_i$. This is possible because a GMM component's variance can approach zero (see~\ref{sec:gaussian-parameterization}). The second subset, consisting of weights with indexes $\notin \gI$, consists of only those weights in the first column of the prompt token embedding table (see~\ref{sec:prompt-embeddings}) corresponding to weights encoding the values of bits on the program tape after $z$, i.e. rows $|z| + 1$ to $R_s$. By construction, these parameters are never ``read'' by the attention head scanning the program tape, so their value does not affect the function being computed by the Transformer. We therefore allow these weights to take on a ``random'' value by setting the corresponding posterior parameters to approximate the Rademacher distribution. This construction allows these bits to be transmitted ``for free'', as we will show in the next section. Formally, the posterior distribution specified by the posterior parameters $\phi^{R,z}$ is therefore:
\begin{equation}
    \text{GMM}(w;\phi^{R,z}_i) = 
    \begin{cases} 
    \delta(w - h_i^{R,z}) &, i \in \gI \\
    \frac{1}{2}\delta(w + 1) + \frac{1}{2}\delta(w - 1) &, i \notin \gI \\
    \end{cases}.
\end{equation}
Therefore, given the posterior parameters $\phi^{R,z}$, the negative log likelihood of the data is equivalent to that under a two-part code with weights $h^{R,z}$:
\begin{align}
-\log_2 p_{M_R}(Y \mid X; \phi^{R,z}) &= 
\mathbb{E}_{h \sim \beta_{M_R}(\cdot;\phi^{R,z})} \left[ -\log_2 p(Y|X; m_{M_R}(h)) \right] \\
&= -\log_2 p(Y|X; m_{M_R}(h^{R,z})) \\
&= -\log_2 p(Y|X; f_T^z).
\end{align}

\paragraph{KL divergence}
We now show that we can select prior parameters $\psi^{R} \in \Psi_{M_R}$ such that the KL divergence term is bounded by $|z| + c_T$, where $c_T$ is a constant depending only on $T$ but not on $z$ or $R$. Because the prior and posterior are parameterized by \emph{independent} GMMs, the overall KL computation factors across the individual weights of the Transformer:
\begin{equation}
    \text{KL}\left[ 
    \beta(\cdot;\phi^{R,z}) \parallel 
    \alpha(\cdot;\psi^{R}) \right] = \sum_i \text{KL}\left[  \beta(\cdot;\phi^{R,z}_i)
    \parallel
    \alpha(\cdot;\psi^{R}_{\texttt{group}(i)})
    \right]
\end{equation}
We consider the weights in the prompt embedding table first, and then discuss the remaining weights.
Recall that a GMM prior is shared for each column in this table. The weights in the first column encodes the contents of the program tape. We set the prior for this first column of the program embedding table to approximate the Rademacher distribution, such that the prior and posterior distribution are the same for the weights corresponding to the $R_s - |z|$ rows in the table after row $|z|$. The KL divergence between two equal distributions is $0$. For the weights in the first $|z|$ rows, the posterior is either a delta function at $1$ or $-1$. In either case, the KL divergence is $1$ bit (see \ref{sec:unimodal-vs-multimodal-priors}, which also highlights the importance of a \emph{multimodal} prior in our construction). For the remaining columns of the prompt embedding table, each row has the same value, and the posterior for all rows within a column approximates a delta function at this value. Therefore, we can set the prior to be equal to the posterior, and the KL divergence is $0$ for these weights. Therefore, the overall KL divergence for the prompt embedding table is $|z|$.

For the finite set of weights outside of the prompt embedding table, the posterior distribution for each weight approximates a delta function. The weights generated by $\zmap$ contain a relatively small set of unique values determined by $T$. We can choose any prior that assigns non-zero probability to these specific weight values. Summing over all fixed weights gives a total constant cost $c_T$ that depends on $T$ but not on $z$ or $R$. 
A single mixture component is formally sufficient, since it can have sufficiently large variance to assign non-zero probability to each unique value in the corresponding weight matrix. However, since the weight values generated by $\zmap$ have values drawn from a finite set determined by $T$, with a sufficient number of components, we can construct the shared GMM priors to be mixtures of delta functions centered at these values, as described in \ref{sec:optimal-mixture-priors}, and therefore $c_T$ can be relatively small.

Combining these parts, the total KL divergence is $|z| + c_T$.

\paragraph{Conclusion}
We have shown that for any program $z \in \gZ_{T,R}$ and resource bound $R \in \gR$, there exists a choice of prior and posterior parameters $(\psi^R, \phi^{R,z})$ such that:
\begin{equation}
L_{M_R}^{\text{adaptive-var}}(Y \mid X, \psi^R, \phi^{R,z}) \leq - \log_2 p(Y \mid X; f_T^z) + |z| + c_T + c_\Psi.
\end{equation}

The minimum codelength for the family $M_R$ is found by minimizing over all $(\psi, \phi)$, so it must be less than or equal to the codelength for this specific choice, minimized over all programs $z$:
\begin{align*}
L^{\bigast,\text{adaptive-var}}_{M_R}(Y \mid X) &= \min_{\psi \in \Psi_{M_R},\phi \in \Phi_{M_R}}~ L^{\text{adaptive-var}}_{M_R}(Y \mid X;\psi,\phi) \\
&\leq \min_{z \in \gZ_{T,R}} L^{\text{adaptive-var}}_{M_R}(Y \mid X;\psi^R,\phi^{R,z}) \\
&\leq \min_{z \in \gZ_{T,R}} \left( c_\Psi + |z| + c_T - \log_2 p(Y \mid X; f_T^z) \right) \\
&\leqc \min_{z \in \gZ_{T,R}} \left( |z| - \log_2 p(Y \mid X; f_T^z) \right)
\end{align*}

This last expression is equal to the definition of $C^{\text{two-part}}_{T,R}(Y \mid X)$:
\begin{align*}
C^{\text{two-part}}_{T,R}(Y \mid X) 
&= \min_{f \in \gF} \left( K_{T,R}(f) - \log_2 p(Y \mid X; f) \right) \\ 
&= \min_{z \in \gZ_{T,R}} \left( |z| - \log_2 p(Y \mid X; f_T^z) \right).
\end{align*}

Therefore our construction satisfies the condition for an asymptotically optimal family of codes.
\end{proof}

\subsection{Analysis of Alternative Two-Part Codes}
\label{sec:quasi-optimal}

Here we provide additional details and discussion related to the bounds in Table~\ref{tab:asymptotic-bounds} introduced in Section~\ref{sec:asymptotic-bounds}. The first row of the table shows a description length bound of $|z| + \log R_s$, leading to an overall codelength bound:
\begin{equation}
    L^{\bigast}_{M_R}(Y \mid X) \leqc C^{\text{two-part}}_{T,R}(Y \mid X) + \log_2 R_s.
\end{equation}

First let us detail the construction of this family of two-part codes in contrast to the construction of the variational code detailed in \ref{sec:transformer-var-code-details}.
One of the key challenges addressed by the adaptive variational code is that we wanted to transmit the $R_s$ prefix token embeddings representing the program tape contents in $|z|$ bits in order to satisfy the conditions of asymptotic optimality. In order for the ``unused'' capacity to be transmitted ``for free'', our proposal required a multimodal GMM posterior that could be set equal to the prior for the $R_s - |z|$ ``unused'' weights. However, implementing and optimizing a multimodal posterior can be challenging. Instead, Alice can adaptively determine the optimal prefix length by selecting $|z| \in \{1, \cdots, R_s \}$, e.g. by sweeping over prefix length as a hyperparameter, using structured dropout, or using some other approach. Alice can then communicate $|z|$ to Bob in $\log_2 R_s$ bits, assuming a naive uniform prior over the integers $\{ 1, \cdots, R_s \}$ for encoding $|z|$. Then, Alice can simply communicate $z$ in $|z|$ bits assuming a Rademacher prior over the $|z|$ weights, as in our previous construction in \ref{sec:transformer-var-code-details}. As in our previous construction, this can be generalized by using an adaptive GMM prior. Alternatively, such a code could potentially be constructed using the quantization method of \citet{han2016deep}, which has also been employed by \citet{lotfi2022pac,lotfi2024non}. Alternatively, we could restrict the prefix embeddings to some fixed vocabulary, akin to discrete prompt optimization. Most directly, in the special case where the vocabulary size is $2$, we reach the same bound.

Next we discuss the bounds in Table~\ref{tab:asymptotic-bounds} resulting from ablating aspects of this recipe. First, if we do not use any form of quantization, and instead use a fixed prior, such as the unimodal Gaussian prior implicit in methods such as weight decay, or a uniform prior implicit in standard MLE objectives, then each weight encoding a bit of $z$ requires more than a single bit to encode. Therefore, the program length term in the bound changes from $|z|$ to $\gO(|z|)$, where the multiplicative factor is $>1$. Second, if we do not adaptively select the prefix length, then we must encode the entire prefix embedding table, which contains $R_s$ rows, regardless of the length of the program it encodes. Therefore, the bound degrades further to $\gO(R_s)$, where $R_s \geq |z|$ as by definition $R_s$ determines the maximum program length that can be represented. Finally, if we do not use layerwise weight sharing, then our bound simply becomes a function of the number of parameters in the Transformer. As $R_t$ specifies the number of layers, we must encode $\gO(R_t)$ weights, in addition to the $\gO(R_s)$ weights in the prefix embedding table. Note that the bounds we have derived assume we are using the Transformer family and specification of $\zmap$ discussed in Section~\ref{sec:two-part-codes-for-transformers}. Alternative constructions could potentially lead to different bounds.

Overall, a weaker bound on the model description length leads to a weaker bound on the minimum of the overall codelength. Therefore, optimizing such an objective may lead to sub-optimal compression, and---from an MDL perspective---sub-optimal model selection. Regardless, the $|z| + \log R_s$ bound is close to the theoretically optimal bound of $|z|$. While pre-trained Transformer decoders are a different setting than the one we have studied here, if our results could be extended to that setting it would provide a strong asymptotic bound for the compression achievable via prompt optimization methods, complementing the positive empirical results from~\citet{akinwande2024understanding}.

\subsection{Scaling Transformer Dimensionality}
\label{sec:mlp-scaling}

Here we provide additional details on the alternative Transformer family discussed in Section~\ref{sec:alternative-families}. In contrast to the Transformer family detailed in \ref{sec:zmap-details}, the program tape contents are implicitly encoded in the MLP parameters. In this construction, we still require prepending $R_s$ ``padding'' tokens to the input in order to scale the context window so that the Transformer can represent Turing machine tapes of size $R_s$. However, the program $z$ is encoded in the MLP layer, and we do not require a prefix embedding table. An ALTA program for populating the program tape contents based on a program encoded in the MLP weights is shown in Figure~\ref{fig:populate_program}. The overall construction first iteratively populates the program tape contents at sequential Transformer positions, and then the Turing machine emulation continues as in our previous construction. Therefore, this construction requires $R_t + R_s + 2$ layers.

\begin{figure}[h!]
\centering
\begin{minted}[
  frame=single,
  framesep=2mm, 
  fontsize=\fontsize{8}{9}\selectfont\ttfamily,
  ]{python}
from alta import program_builder as pb

def build_program_spec(program_input: list[int]) -> pb.ProgramSpec:
  """Returns ALTA program spec for populating program tape."""

  variables = {
      "done": pb.var(2),
      "is_start": pb.input_var(2, init_fn=lambda x: x == START),
      "head": pb.input_var(2, init_fn=lambda x: x == START),
      "symbol": pb.input_var(2),
      "program_index": pb.var(len(program_input)),
  }
  attention_heads = {
      "head_left": pb.v_relative("head", -1),
  }

  def ffn_fn(x):
    if x["program_index"] == len(program_input):
      x["done"] = 1
      return

    if x["head"]:
      x["symbol"] = int(program_input[x["program_index"]])
      x["head"] = 0

    if not x["is_start"] and x["head_left"]:
      x["head"] = 1

    x["program_index"] = x["program_index"] + 1

  return pb.program_spec(
      ffn_fn=ffn_fn, variables=variables, heads=attention_heads,
      output_name="symbol", input_range=2, position_range=None,
  )
\end{minted}
\caption{ALTA program specification for populating a program tape given a program encoded in the MLP layer.}
\label{fig:populate_program}
\end{figure}

The MLP matrices generated by the ALTA program in Figure~\ref{fig:populate_program} require $\gO(|z|^2)$ parameters, and have rank $\gO(|z|)$, where $|z|$ is the program length. While the existence of a mapping analogous to $\zmap$ for this Transformer family indicates that asymptotically optimal description length objectives exist, implementing such a prior introduces complex dependencies across the parameters of the MLP. Given the rank requirements on the MLP matrices under this construction, rank factorization would not be an effective means of compression, but it would be interesting for future work to consider alternative constructions or practical methods for matrix compression that could yield asymptotic bounds approaching the theoretical ideal.

\section{Additional Experimental Results and Details}
\label{sec:appendix-experiments}

Here we provide additional details and results for the Transformer and MLP experiments discussed in Section~\ref{sec:experiments}.

\subsection{Transformer Experiments}
\label{sec:transformer-exp-details}

\paragraph{Parity task details} We follow the setting of \citet{shaw2024alta}. The train set consists of examples with lengths ranging from 1 to 20. The train set includes $100,000$ examples, with roughly an equal number of examples per number of ones. We evaluate out-of-distribution (OOD) accuracy on a test set with lengths ranging from 21 to 40. 

\paragraph{ALTA program for parity} The ALTA program for parity that we use for our manual initialization is based on the sequential algorithm with relative position representations presented in~\citet{shaw2024alta}. This algorithm computes parity by iterating through each position (one per layer), flipping a running parity
bit every time a one is encountered. We simply reverse the direction of iteration from left-to-right to right-to-left, to accommodate that our model's classification decision depends on the \texttt{SEP} token at the beginning of the input sequence (see below).

\paragraph{Architecture and ALTA weight conversion}
One implementation challenge is that our experiments require a trainable Transformer that is compatible with ALTA-compiled weights. The reference ALTA Transformer implementation in \texttt{numpy} is not trainable~\citep{shaw2024alta}. To accomplish this, we implement a trainable version in Jax~\citep{jax2018github}, following the same parameterization expected by the ALTA compiler. This parameterization uses relative position representations~\citep{shaw2018self} following the parameterization of T5~\citep{raffel2020exploring}, which uses a single scalar bias term for relative positions within some window. It also uses an alternative parameterization for the attention head output transformations, as proposed by~\citet{elhage2021mathematical}, to simplify compilation; however, this alternative attention output parameterization is shown to be equivalent to that of the original Transformer. We select the logits for the \texttt{SEP} token as the output of the model, as our experiments focus on classification tasks.

A more significant challenge is that ALTA-compiled weights are not designed to be used with layer normalization~\citep{ba2016layer}. As the sequential algorithm for parity requires at least as many layers as bits in the longest input sequence, and our test set contains sequences up to length $40$, we require a relatively deep Transformer, using $42$ layers in our experiments. In our initial experiments we found that while a shallow Transformer could converge without layer normalization, training a Transformer with $\geq 40$ layers was not feasible on the parity task without some form of normalization. We also evaluated using tanh normalization, as an alternative to layer normalization, as proposed by~\citet{zhu2025transformers}, which we found to perform at least as well as layer normalization in our experiments, with or without using the variant with a dynamic scalar. We therefore use tanh normalization, and adapted the ALTA compiler to generate weights compatible with this form of normalization.

This required a few changes to the ALTA compiler. However, as our parity program only requires categorical variables, the changes are relatively straightforward. First, we changed how categorical variables are represented in the residual stream. In the original ALTA compiler, these are represented as ``one-hot'' vectors, i.e. a sequence of $0$s and $1$s. Instead, we represent categorical variables as ``signed one-hot'' vectors, i.e. a sequence of $-1$s and $1$s, with the position of the $1$ still representing the value of the variable. This change only requires minor modifications to the compiled embedding and MLP parameters. We additionally scale the compiled parameters by configurable scalars $>1$ to ensure the outputs of every sub-layer sufficiently saturate the tanh function. This also makes the compiled parameters more robust to noise. We zero-pad the ALTA-generated weights up to the dimensions specified by the given hyperparameters. Finally, as we are using only categorical variables, we can use the more standard ReLU activation in the MLP layer, as opposed to the clipped ReLU activation required to handle numerical variables in the original ALTA compiler.

We can then convert from the scalar weights produced by the ALTA compiler to posterior distributions used by our experiments with variational codes by setting the mean of the posterior distribution to the scalar weight and setting the variance to be a small, configurable scalar (e.g. $\nu = -10$). As ALTA-generated weights contain few unique values, we can analytical determine the optimal parameters for the adaptive prior (\ref{sec:optimal-mixture-priors}).

We otherwise use a ``tiny'' Transformer encoder, which was previously shown to be sufficient to fit the parity task~\citep{shaw2024alta}, with $2$ attention heads, $128$ model dimensions, and $512$ hidden MLP dimensions. We prepend $20$ prompt tokens to the model input.

\paragraph{Relation to asymptotic bounds} The experiment setting, using a relatively small model on a simple synthetic task, is potentially not well aligned with where the asymptotic guarantees from our theory are most likely to apply, which is in the limit as model size and dataset complexity increase. Regardless, this experiment setting is a useful initial investigation of the proposed variational code. Relatedly, the ALTA program for parity does \emph{not} emulate a Turing machine; rather it represents the algorithm within the weights of the MLP and attention sub-layers and leaves the prompt tokens effectively unused. However, this is not entirely at odds with our theory; as discussed in \ref{sec:practical-considerations}, the minimizer of an asymptotically optimal code may be quite different from the Turing machine emulation used to prove the asymptotic bound, especially under finite resource constraints.

\paragraph{Model training}

We use Monte-Carlo sampling to estimate the expected data likelihood in the adaptive variational codelength objective during training. 
Each gradient step is computed over 2 MC weight samples shared across a minibatch of 128 examples, for a total of 256 forward passes of the model (computed in parallel) per step. Prior work has commonly shared a single MC weight sample across a minibatch. In our initial experiments, we did not find significant improvements from increasing beyond 2 MC weight samples per minibatch. For the posteriors corresponding to weights in the prompt token embeddings table, we use a GMM posterior with 2 components, and use Gumbel-Softmax~\citep{jang2017categorical,maddison2017concrete} sampling with a temperature of $0.1$. For all other posteriors, we use a single Gaussian, and use the standard Gaussian reparameterization trick~\citep{kingma2014auto}. We did not implement the local reparameterization trick of~\citet{kingma2015variational} for reducing variance, although this could be useful for future work.

Following~\citet{blundell2015weight}, we also use MC samples to estimate KL divergence, as there is no closed form expression for the KL divergence of GMMs with multiple components. Specifically, we use $100$ MC samples for estimating KL.
For sampling from posterior distributions consisting of mixtures with more than one component, we use the Straight-Through (ST) Gumbel-Softmax estimator~\citep{jang2017categorical}, as soft estimates can be problematic and lead to negative KL estimates. For example, if the components of the true posterior distribution have relatively low variance, then samples from the Gumbel-Softmax distribution with a sufficiently large temperature can actually have higher probability under a higher-variance prior distribution than under a prior matching the posterior.

We use the Adam optimizer~\citep{kingma2014adam}, with a learning rate of $10^{-3}$, $1000$ warmup steps, and $50,000$ total steps with exponential decay of the learning rate. We determined some hyperparameters such as learning rate on an in-distribution validation set. 

As we train the model on minibatches consisting of $128$ samples from the $100,000$ training examples, we need to scale the coefficient of the KL term in the loss to accommodate that the data likelihood estimate is over a small sample, $\approx 10^{-3}$, of the full dataset. Therefore, we scale the KL term in the loss by a coefficient of $10^{-3}$, which we found to produce the lowest total codelength. Sweeping the KL coefficient produces a tradeoff between the data likelihood and the KL divergence, as observed in Table~\ref{tab:mlp_results} for the MLP experiments.

\paragraph{Random initialization}
For the MLE baseline experiment, we initialize the weight matrices using the standard variant of normal random initialization proposed by~\citet{he2015delving}, as our network uses ReLU activations. Bias vectors are initialized to zeros.
For the experiments with random initialization and the variational objective, there are many possibilities for initializing the parameters of the prior and posterior GMMs, and we did not explore this space exhaustively. However, we did find in our initial experiments that it was necessary to initialize the posterior distributions with relatively low variance in order for the model to converge towards a solution that fits the data. Therefore, for the means of the posterior distributions, we sample random weight values following the same method as the MLE experiments, set these sampled weights as the GMM means, and then set the GMM variance parameter $\nu$ (i.e. the variance prior to the softplus function) to be $-10$. We initialize the prior parameters similarly, but set the initial variance parameter $\nu$ to be $1$ to avoid instability in KL estimates at initialization. We set all mixing weights for the GMMs to be initially equal. Future work could explore alternative initializations.

\paragraph{Optimizer Comparison} We evaluated SGD in comparison to Adam as the optimizer for Transformers with the variational objective for the parity task, following the setting reported in Section~\ref{sec:experiments}.

\begin{table}[h]
    \centering
    \caption{Comparison of Adam and SGD for optimizing variational objective.}
\vspace{0.15cm}
\scalebox{0.85}{
    \label{tab:optimizer_results}
    \begin{tabular}{lccc}
        \toprule
        \textbf{Optimizer} & \textbf{Train NLL} (bits) & \textbf{KL} (bits) & \textbf{OOD Accuracy} \\
        \midrule
        Adam & $2.36 \times 10^{2}$ & $2.81 \times 10^{6}$ & $60.4\%$ \\
        SGD  & $1.58 \times 10^{5}$ & $2.44 \times 10^{6}$ & $49.6\%$ \\
        \bottomrule
    \end{tabular}
}
\end{table}

Adam and SGD exhibit different optimization behavior, and a different trade-off between optimizing the model cost (KL divergence) and data cost (negative log likelihood). Both methods fail to find a codelength comparable to that of the manual initialization. SGD more severely underfits the data, leading to poor generalization.
This highlights the sensitivity of optimizing the variational objective to the specific optimization procedure. Future work could more fully explore alternative hyperparameters, learning rate and KL coefficient schedules, or alternative families of optimizers, such as approximate second-order methods, e.g., K-FAC~\citep{martens2015optimizing}.

\subsection{MLP Experiments}
\label{sec:mlp-experiments}

Here we provide additional details on the MLP experiments discussed in Section~\ref{sec:experiments}.
To understand why our Transformer models underfit the proposed objective when randomly initialized (as evidenced by the significantly lower loss achieved by ALTA-initialized models), we study the simplified setting of a 2-layer MLP. Specifically, we study the task of learning an identity function over a vector of 4 binary values. The model is a MLP with 4 input dimensions, 16 hidden dimensions, and 4 output dimensions, which makes 4 independent binary classifications.
We can again compare random initialization with initializing the model with manually chosen weights that fit the data and have low complexity according to the proposed objective, by selecting weights inspired by how the ALTA compiler generates MLP layers.

\begin{figure}[t]
\vspace{-0.55cm}
    \centering
    \includegraphics[width=0.85\columnwidth,keepaspectratio]{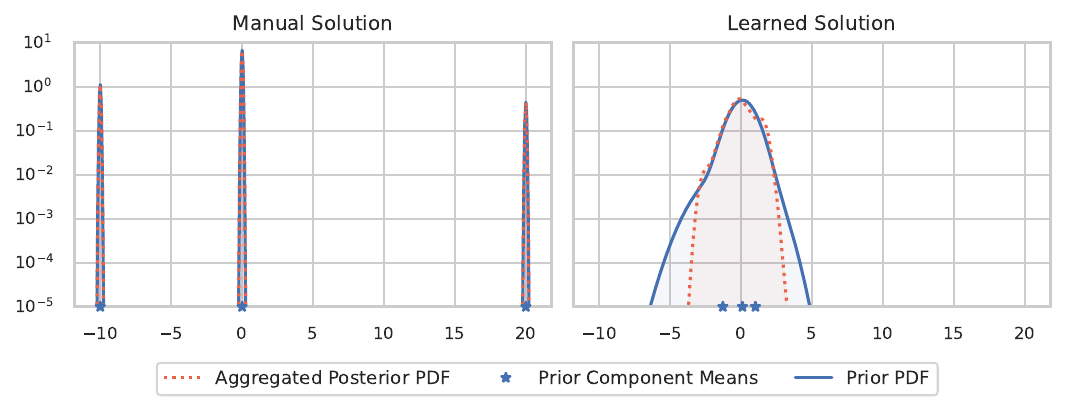}
    \caption{Manually specified vs. learned distributions for MLP trained with a variational objective on the identity task with an adaptive GMM prior. The aggregated posterior shown in the figure is an equally-weighted mixture of the individual Gaussian posteriors for each MLP weight.} %
    \label{fig:mlp_plots}
\end{figure}

\begin{table}[h!]
\centering
\caption{MLP performance on identity task.}
\vspace{0.15cm}
\scalebox{0.85}{
\label{tab:mlp_results}
\begin{tabular}{
  llccccc
}
\toprule
{\textbf{Objective}} &
{\textbf{Init.}} & {\textbf{KL Coef.}} & {\textbf{KL}} (bits) & {\textbf{Train NLL}} (bits) & {\textbf{Codelength}} (bits) & {\textbf{Accuracy}} \\
\midrule
        MLE Baseline & Random & --- & --- & $8.82$ & --- & $100\%$ \\
        \midrule
        Variational & Random & $1.0$ & $2.87$ & $2.64 \times 10^{3}$ & $2.64 \times 10^{3}$ & $54\%$ \\
                    &        & $10^{-1}$ & $3.19 \times 10^{1}$ & $1.90 \times 10^{3}$ & $1.93 \times 10^{3}$ & $63\%$ \\
                    &        & $10^{-2}$ & $2.24 \times 10^{2}$ & $1.12 \times 10^{2}$ & $3.37 \times 10^{2}$ & $100\%$ \\
                    &        & $10^{-3}$ & $2.88 \times 10^{2}$ & $5.39 \times 10^{1}$ & $3.42 \times 10^{2}$ & $100\%$ \\
                    &        & $10^{-4}$ & $2.98 \times 10^{2}$ & $5.09 \times 10^{1}$ & $3.49 \times 10^{2}$ & $100\%$ \\
        \midrule
        Variational & Manual & $10^{-1}$ & $9.87 \times 10^{1}$ & $ 0.0$ & $9.87 \times 10^{1}$ & $100\%$ \\
        \bottomrule
        
\end{tabular}
}
\end{table}

The results are shown in Table~\ref{tab:mlp_results}, and show similar trends as the Transformer experiments. We see that when starting from a random initialization, the optimization process fails to find a loss comparable to that achieved via manual initialization, indicating poor optimization of the proposed objective. To understand why the randomly initialized models underfit, we can inspect the properties of the learned prior and posterior distributions compared to those from manual initialization. Notably, the prior distribution appears to converge to a unimodal distribution, as visualized in Figure~\ref{fig:mlp_plots}, in contrast to the multimodal prior of the manual solution, with lower-variance components.

\paragraph{Task and network}

Our network is a $2$ layer MLP with $4$ inputs dimensions and $16$ hidden dimensions, that computes the function:
\begin{equation}
    f(x;\theta) = \text{sigmoid}( W^2~\text{ReLU}(W^1(x) + b^1) + b^2)
\end{equation}
with parameters $\theta = (W^1$, $b^1$, $W^2$, $b^2$). 

Our goal is to learn the identity function such that $f(x;\theta) \approx x$ for any binary vector $x \in \{0, 1\}^4$.

\paragraph{Posterior and prior distributions} We parameterize the posterior distribution with an independent Gaussian corresponding to each weight in the MLP. We use a single adaptive GMM prior with 3 components shared across all weights.

\paragraph{Manual solution} A relatively simple ``manual'' solution to the identity task can be constructed by choosing some scalar $\lambda \gg 0$ (we choose $\lambda=20$ in our experiments) and defining the network weights as follows:
\begin{align*}
    W^2 = (W^1)^\top = 
\begin{bmatrix}
\lambda & 0 & 0 & 0 & 0 & \cdots & 0 \\
0 & \lambda & 0 & 0 & 0 & \cdots & 0 \\
0 & 0 & \lambda & 0 & 0 & \cdots & 0 \\
0 & 0 & 0 & \lambda & 0 & \cdots & 0 \\
\end{bmatrix}
&&
    b^1 =
\begin{bmatrix}
-\frac{\lambda}{2} \\
\vdots \\
-\frac{\lambda}{2}
\end{bmatrix}
&&
    b^2 =
\begin{bmatrix}
-\frac{\lambda}{2} \\
\vdots \\
-\frac{\lambda}{2}
\end{bmatrix}
\end{align*}

\paragraph{Training details} The negative log-likelihood of $y$ given $x$ with respect to $\theta$ follows the standard binary cross-entropy loss:
\begin{equation}
    -\log_2 p(y \mid x;\theta) = \sum_{i=1}^4 -y_i \log_2(f(x;\theta)_i) - (1 - y_i) \log_2(1 - f(x;\theta)_i)
\end{equation}
We train the MLP with a variational objective. We specify posterior and prior distributions over $\theta$ as specified above, minimizing the variational objective consisting of the sum of the KL divergence between posterior and prior  distributions, and the expected binary cross entropy loss with respect to sampling weights $\theta$ from the posterior distribution.
We over-sample the space of $2^4$ possible binary vectors of length $4$ to create minibatches for training of $128$ examples.

Sampling, initialization, and training details are similar to the Transformer setting detailed above, except training only required $1,000$ steps. We show a sweep over different KL coefficients in Table~\ref{tab:mlp_results}.

\section{Extended Related Work}
\label{sec:extended-related-work}

Here we offer an extended discussion of related work previously summarized in Section~\ref{sec:related-work}.

\subsection{Theoretical Foundations} Our work is closely related to the theoretical notion of \emph{universal induction}~\citep{solomonoff1964formal,hutter2000theory,lattimore2013no,achille2021information} -- and related resource-bounded variants~\citep{levin1973universal,nakkiran2021turing}. Previous work has developed theoretical frameworks applying these notions to density estimation~\citep{barron1985logically,barron1991minimum}, sequential decision theory~\citep{hutter2005universal}, and kernel methods~\citep{hamzi2024bridging}.  For neural networks, \citet{schmidhuber1997discovering} introduced a probabilistic search algorithm for discovering neural networks with low Kolmogorov complexity. However, none of these theoretical frameworks directly lead to practical and scalable training objectives for neural networks.

\subsection{Variational Inference}
 Variational inference was proposed as a regularizer to improve generalization~\citep{hinton1993keeping}, under the ``bits back'' argument, which can also be viewed as a form of Bayesian inference~\citep{honkela2004variational}. 
Such variational methods were popularized by the success of variational autoencoders (VAEs)~\citep{kingma2014auto}, and related advances that enabled applying such methods to neural networks using standard gradient-based optimizers~\citep{graves2011practical,blundell2015weight}. Variational methods have found success for network compression -- enabling weight pruning~\citep{louizos2017bayesian} and quantization~\citep{achterhold2018variational,ullrich2017soft} -- as well as for probing~\citep{voita2020information} and improving uncertainty modeling~\citep{sankararaman2022bayesformer,gal2016dropout}. Other work has pursued applying variational bottlenecks to network \emph{activations}~\citep{fehr2024nonparametric}, as opposed to network \emph{weights}.
However, in general, the effectiveness of variational methods for implementing MDL-inspired regularizers for improving generalization has remained elusive~\citep{blier2018description,cinquin2021pathologies}. Prior work has typically found tighter compression bounds via prequential coding methods~\citep{blier2018description,bornschein2023sequential}. Our work provides some new perspectives on the poor performance of variational approaches: objectives evaluated by prior work may have poor asymptotic bounds, and variational objectives can be difficult to optimize. 

\subsection{Non-Variational Complexity Measures}\label{sec:non-variational-related-work} As regularization is a foundational concept in machine learning, there are many related methods for neural networks~\citep{jiang2020fantastic}, and we give only a brief survey here. We focus on recently proposed compression-based methods with connections to the MDL principle, as these are most relevant to the description length objectives discussed in this paper.

Motivated by prior work on the intrinsic dimensionality of neural networks~\citep{li2018measuring}, \citet{lotfi2022pac} propose a complexity measure based on subspace sampling combined with dynamic quantization, similar to the quantization method of~\citet{han2016deep}. 
\citet{lotfi2024non} combines the method of \citet{lotfi2022pac} with LoRA~\citep{hu2022lora,dettmers2023qlora}, to further drive compression. Both \citet{lotfi2022pac} and \citet{lotfi2024non} discuss generalization bounds for neural networks related to compression, a relevant topic, but not one we study explicitly in this paper.
\citet{demoss2025complexity} studies the connection between the complexity of a neural network and ``Grokking'', i.e. the transition from memorization to generalization during training. To this end, they propose a new complexity measure and regularizer based on coarse-grained quantization and spectral entropy, which encourages low rank parameter matrices.
\citet{abudy2025minimum} similarly advocate for regularizers grounded in MDL instead of standard norm-based penalties on weights. They demonstrate how norm-based regularizers actively push RNN weights away from perfect initializations on algorithmic tasks, while their MDL-inspired objective does not. Their proposed MDL measure can be interpreted as a two-part code, using a non-differentiable prior that encodes individual weight values according to their representation as a rational number.
\citet{dwivedi2023revisiting} similarly studies MDL-inspired complexity measures in over-parameterized models where parameter count does not provide a suitable measure of complexity, and proposes a principled measure, MDL-COMP. However, their scope is limited to specific classes of linear models and kernel methods.

The variational code proposed in Section~\ref{sec:transformer-practical-code} drives compression by encouraging soft quantization of model weights around the means of the components of the GMM prior, and by reducing the effective number of weights by encouraging higher uncertainty posterior distributions close to the prior for ``unused'' capacity. The non-variational approaches discussed here similarly combine some form of quantization with some alternative means of reducing the effective number of parameters, e.g. through low-rank approximation. While prior work has not established asymptotic guarantees for these alternative encoding schemes, future work could aim to construct families of asymptotically optimal, or quasi-optimal, codes based on these alternative methods.

\subsection{Computational Universality of Transformers}
\label{sec:zmap-related-work}

There has been considerable interest in establishing the theoretical expressivity of various classes of Transformers~\citep{perez2021attention,chiang2023tighter,yun2020are,feng2023towards,merrill2024expressive,nowak2024representational,merrill2024little}. Most relevant to our constructed mapping between computable, rational-valued distributions and Transformer weights is prior work establishing various equivalences between Transformers and Turing machines. \citet{perez2021attention} established Turing completeness of an encoder-decoder Transformer in a non-probabilistic language recognition setting, with \citet{merrill2024expressive} and \citet{feng2023towards} extending this result to decoder-only variants. Most relevant to our result is \citet{nowak2024representational}, which demonstrates Turing completeness in a probabilistic setting. One difference is that our result holds for Transformer encoders (in the limit as context size and number of layers increase) compared with \citet{nowak2024representational}'s result for Transformer decoders with intermediate decoding steps. Their result also does not directly extend to prefix Turing machines, which is necessary to support our theoretical results. However, their work could be potentially helpful for future work seeking to establish the existence of families of asymptotically optimal codes for Transformer decoders with intermediate decoding steps.
Finally, \citet{schuurmans2023memory} shows that a Transformer decoder augmented with an external memory tape is computationally universal. Broadly, considering description length objectives over models with external tools could be of interest for future work. 

We use the ALTA compiler~\citep{shaw2024alta} to support our demonstration of universality. ALTA was inspired by RASP~\citep{weiss2021thinking}, an alternative language for compiling programs to Transformers, and the related Tracr~\citep{lindner2024tracr} compiler. ALTA offers better support for programs with loops, which is useful for emulating a Turing machine.

\subsection{Simplicity Bias} While our work focuses on implementing a simplicity bias \emph{explicitly} in a training objective, several studies have evaluated the \emph{implicit} simplicity bias of neural networks~\citep{zhou2023algorithms,abbe2023generalization,bhattamishra2023simplicity,tsoy2024simplicity,chen2024sudden,mingard2025deep}, or lack thereof~\cite[e.g.,][]{nikankin2025arithmetic}. Others study the simplicity bias of \emph{in-context learning} with pre-trained models~\citep{goldblum2023no,deletang2024language}, or proposed tasks to enhance this bias~\citep{grau-moya2024learning}.

\end{document}